\documentclass{article}

\PassOptionsToPackage{numbers, compress}{natbib}

\usepackage[final]{neurips_2024}

\usepackage{microtype}
\usepackage{graphicx}
\usepackage{subfigure}
\usepackage{booktabs} 
\usepackage{hyperref}

\usepackage{multirow}
\usepackage{amssymb}
\usepackage{algorithmic}
\usepackage{mathtools}
\usepackage{amsthm}
\usepackage{amsmath}
\usepackage[capitalize,noabbrev]{cleveref}
\usepackage[utf8]{inputenc} 
\usepackage[T1]{fontenc}    
\usepackage{booktabs}       
\usepackage{nicefrac}       
\usepackage[table]{xcolor}         
\usepackage{latexsym}
\usepackage{multirow}
\usepackage{makecell}
\usepackage{enumitem}
\usepackage{graphicx}
\usepackage{setspace}
\usepackage{subfigure}
\usepackage{latexsym}
\usepackage{inconsolata}
\usepackage{bm}
\usepackage{wrapfig}
\usepackage{arydshln}
\usepackage{hyperref}
\usepackage{url}
\usepackage{pifont}
\usepackage{bbding}
\usepackage[misc]{ifsym}
\usepackage{microtype}
\usepackage{setspace}
\usepackage{hyperref}

\definecolor{cpgcolor}{hsb}{0.7, 0.1, 0.98}

\newcommand{\cmark}{\ding{51}}%
\newcommand{\xmark}{\ding{55}}%

\newtheorem{lemma}{Lemma}
\newtheorem{theorem}{Theorem}

\title{Advancing Spiking Neural Networks for Sequential Modeling with Central Pattern Generators}

\author{%
\textbf{Changze Lv}$^{1}$\footnotemark[1] \quad
\textbf{Dongqi Han}$^{2}$\footnotemark[2] \quad
\textbf{Yansen Wang}$^{2}$\footnotemark[2] \quad
\textbf{Xiaoqing Zheng}$^{1}$\footnotemark[2]\\
\textbf{Xuanjing Huang}$^{1}$ \quad
\textbf{Dongsheng Li}$^{2}$ \\
$^1$School of Computer Science, Fudan University \\ 
$^2$Microsoft Research Asia \\
\texttt{\{czlv22\}@m.fudan.edu.cn},
\texttt{\{zhengxq,xjhuang\}@fudan.edu.cn},\\
\texttt{\{yansenwang,dongqihan,dongsli\}@microsoft.com} \\
}

\begin{document}

\maketitle

\footnotetext[1]{The work was conducted during the internship of Changze Lv at Microsoft Research Asia.}
\footnotetext[2]{Corresponding authors.}

\begin{abstract}
Spiking neural networks (SNNs) represent a promising approach to developing artificial neural networks that are both energy-efficient and biologically plausible.
However, applying SNNs to sequential tasks, such as text classification and time-series forecasting, has been hindered by the challenge of creating an effective and hardware-friendly spike-form positional encoding (PE) strategy.
Drawing inspiration from the central pattern generators (CPGs) in the human brain,  which produce rhythmic patterned outputs without requiring rhythmic inputs, we propose a novel PE technique for SNNs, termed CPG-PE.
We demonstrate that the commonly used sinusoidal PE is mathematically a specific solution to the membrane potential dynamics of a particular CPG.
Moreover, extensive experiments across various domains, including time-series forecasting, natural language processing, and image classification, show that SNNs with CPG-PE outperform their conventional counterparts.
Additionally, we perform analysis experiments to elucidate the mechanism through which SNNs encode positional information and to explore the function of CPGs in the human brain.
This investigation may offer valuable insights into the fundamental principles of neural computation.
Our code is available at \url{https://github.com/microsoft/SeqSNN}.
\end{abstract}

\section{Introduction} \label{sec:introduction}

Spiking neural network (SNN) \citep{Maas1997NetworksOS} has increasingly attracted research interests in recent years, primarily due to its energy efficiency, event-driven paradigm, biological plausibility, and other distinctive properties.
The spiking neurons in SNN are dynamical systems that generate binary signals (spike or non-spike) and communicate these signals like artificial neural networks (ANNs) for computation \citep{Fang2021DeepRL,Ding2021OptimalAC,Zhou2022SpikformerWS,yao2023spike,yao2024spikedriven,Fang2020IncorporatingLM,lv2023spiking,li2024seenn}.
Many advanced architectures and methodologies developed for ANNs are also applicable to SNNs, enhancing their capabilities.
Notable among these are backpropagation \citep{wu2018spatio}, batch normalization \citep{zheng2021going,duan2022temporal}, and Transformer architecture \citep{Zhou2022SpikformerWS, zhou2023spikingformer,yao2023spike,yao2024spikedriven}, which collectively broaden the functional scope of SNNs.



Despite the promising advances in SNNs, several challenges persist when adapting them to diverse tasks.
A fundamental challenge is that SNNs, which are event-triggered, lack robust and effective mechanisms to capture indexing information, rhythmic patterns, and periodic data.
This limitation can adversely affect SNNs' ability to process and analyze different data modalities, including natural language, and time series.
Meanwhile, while SNNs aim to emulate the neural circuits of the brain, their reliance on spike-based communication imposes limitations. 
Consequently, not all deep learning techniques applicable to ANNs can be directly transferred to SNNs.
For instance, methods like HiPPO \citep{gu2020hippo} or trigonometric positional encoding \citep{Vaswani2017AttentionIA} are not readily compatible with the spike format used in SNNs.
Moreover, even the most state-of-the-art ANNs still lag significantly behind human capabilities in many tasks \citep{zador2019critique,mitchell2021ai}.
Therefore, to enhance the functionality of SNNs, one promising approach is to draw further inspiration from biological neural mechanisms.  In this regard, we propose the analogy of central pattern generators (CPGs) \citep{marder2001central}, a kind of neural circuit well-known in neuroscience, with positional encoding (PE), a technique extensively utilized in deep learning.
This analogy is designed to operate within the SNN framework, potentially bridging the gap between biologically inspired models and modern deep learning techniques.

In neuroscience, a CPG (See \Cref{fig:method_cpgpe} for an illustration) is a group of neurons capable of producing rhythmic patterned outputs without requiring rhythmic inputs \citep{marder1996principles, grillner2006biological}.
These neural circuits are found in the spinal cord and brainstem and are responsible for generating the rhythmic signals that control vital activities such as locomotion, respiration, and chewing \citep{kiehn2016decoding}.

On the other hand, PE is an important technique for ANNs, particularly within models tailored for sequence processing task \citep{Vaswani2017AttentionIA,liu2022petr,dosovitskiy2021an}.
By endowing each element of the input sequence with positional information, typically achieved through diverse mathematical formulations or learnable embeddings, neural networks acquire the capability to discern the order and relative positions of the elements within the sequence.

We argue that these two concepts, despite seemingly unrelated, can be connected profoundly. Intuitively, CPG and PE both generate periodic outputs (with respect to time for CPG and with respective to position for PE).
Moreover, in this paper, we reveal a deeper relationship between these two concepts by showing that \textbf{the widely used sinusoidal PE is mathematically a particular solution of the membrane potential dynamics of a specific CPG}.
However, current SNN architectures exhibit a notable deficiency in implementing an effective and biologically plausible PE mechanism.
Existing so-called positional encoding methods for SNNs \citep{Zhou2022SpikformerWS,yao2023spike} rely on input data, often resulting in non-spike and repetitive outputs for different positions.
Furthermore, incorporating PE techniques designed for ANNs necessitates the calculation of membrane potentials, which is incompatible with the spike format of SNNs.
To address these issues, we draw inspiration from the spiking properties of the CPGs and propose a straightforward yet versatile PE technique for SNNs, termed CPG-PE.
This method encodes positional information with multiple neurons with various patterns of spike trains.
To summarize the highlights of our study:
\begin{itemize}
\vspace{-1.5mm}
\item \textbf{Novel Positional Encoding for SNNs.} We introduce a bio-plausible and effective PE approach tailored specifically for SNNs. This innovative strategy draws inspiration from the central pattern generator found in the human brain. Additionally, we propose a straightforward implementation of CPG-PE in SNNs, which is also compatible with neuromorphic hardware as it can be realized using circuits of leaky integrate-and-fire neurons.

\item \textbf{Consistent Performance Gain.} Our proposed methods significantly and consistently enhance the performance of SNNs across a wide range of sequential tasks, including time-series forecasting and text classification.

\item \textbf{Insightful Analysis.} Our research represents one of the pioneering efforts to comprehensively analyze (1) the mechanism by which SNNs capture positional information and (2) the role of CPGs in the brain. This analysis provides valuable insights into the underlying principles of neural computation.
\end{itemize}

\section{Preliminaries}
\subsection{Spiking Neural Networks}


The basic unit in SNNs is the spiking neuron, such as the leaky integrate-and-fire (LIF) neuron \citep{Maas1997NetworksOS}, which operates based on an input current $I(t)$ and influences the membrane potential $U(t)$ and the spike $S(t)$ at time $t$.
The dynamics of the LIF neuron are described by the following equations:
\begin{align}
\label{equ:membranePotential}
&U(t)=H(t-\Delta t)+I(t), \quad I(t)=f(\mathbf{x}; \mathbf{\theta}), \\
\label{equ:ht}
&H(t)=V_{reset}S(t) +\left(1-S(t)\right)\beta U(t),\\
\label{equ:s(t)}
& S(t)=
\begin{cases}
1, & \text{if  $U(t) \geq$ $U_{\rm thr}$} \\ 
0, & \text{if  $U(t) <$ $U_{\rm thr}$} 
\end{cases},
\end{align}
Here, $I(t)$ is the spatial input to the LIF neuron at time step $t$, calculated using the function $f$ with $\mathbf{x}$ as input and $\mathbf{\theta}$ as learnable parameters.
$\Delta t$ is the discretization constant that determines the granularity of LIF modeling, and $H(t)$ is the temporal output of the neuron at time step $t$.
The spike $S(t)$ is defined as a Heaviside step function based on the membrane potential.
When $U(t)$ reaches the threshold $U_{\rm thr}$, the neuron fires, emitting a spike, and the temporal output $H(t)$ resets to $V_{reset}$.
If the membrane potential $U(t)$ does not reach the threshold, no spike is emitted, and $U(t)$ decays to $H(t)$ at a decay rate of $\beta$.

In this paper, we choose direct training with surrogate gradients as our method to train SNNs.
we follow \cite{SpikingJelly} to choose the arctangent-like surrogate gradients as our error estimation function when backpropagation, which regards the Heaviside step function as:
$S(t) \approx \frac{1}{\pi} \arctan(\frac{\pi}{2}\alpha U(t))+\frac{1}{2}$
, where $\alpha$ is a hyper-parameter to control the frequency of the arctangent function.
Therefore, the gradients of $S$ are 
$\frac{\partial S(t)}{\partial U(t)}=\frac{\alpha}{2} \frac{1}{(1+(\frac{\pi}{2}\alpha U(t))^{2})}$
and thus the overall model can be trained in an end-to-end manner with back-propagation through time (BPTT)  \citep{Werbos1990BackpropagationTT}.

\begin{figure*}[t]
\centering
\includegraphics[width=0.99 \textwidth]{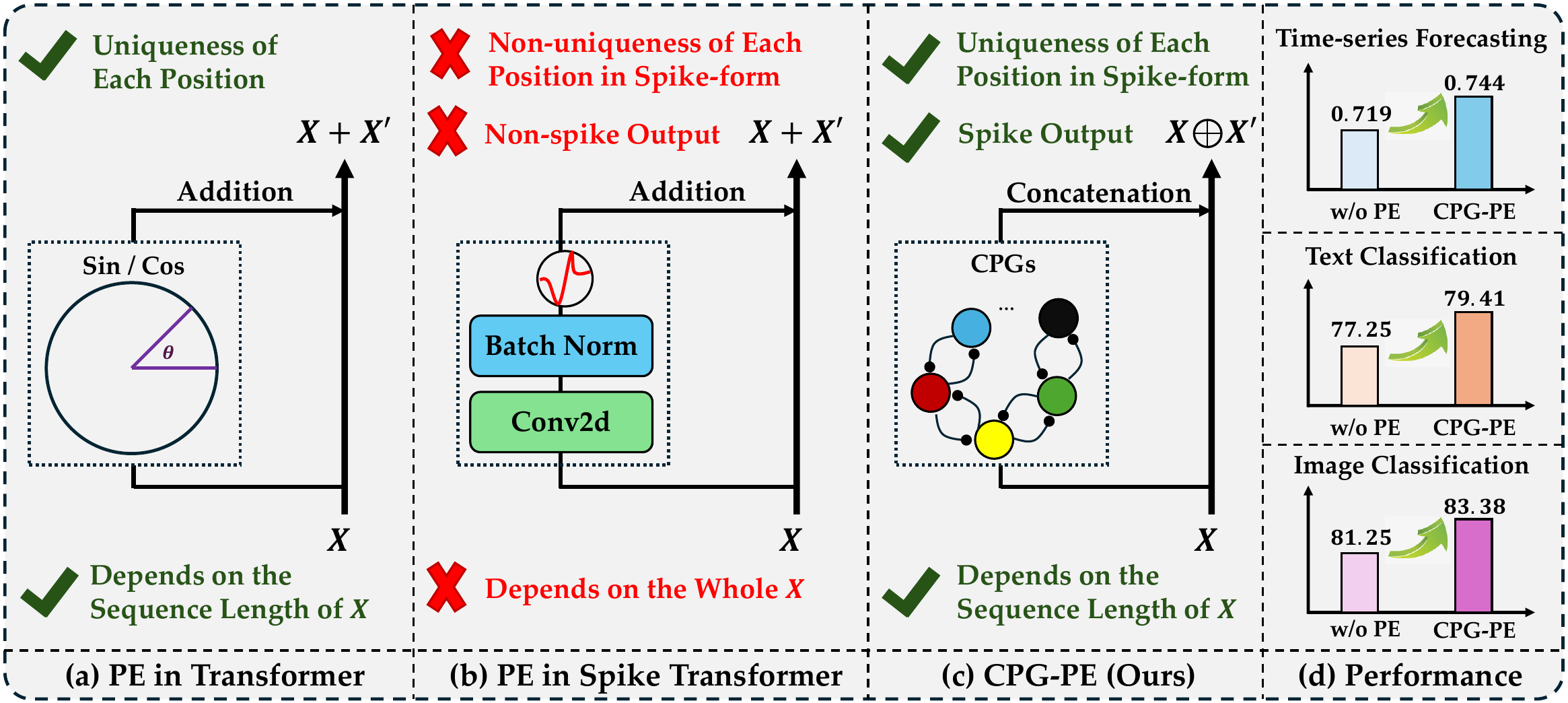}
\caption{\label{fig:intro}
(a) Positional encoding (PE) in ANN Transformers.
(b) Relative PE \protect\footnotemark[3] in Spike Transformers \citep{Zhou2022SpikformerWS,yao2023spike,yao2024spikedriven}.
(c) Our Proposed CPG-PE method.
(d) CPG-PE consistently improves learning performance across various tasks.
CPG-PE is an ideal PE method tailored for SNNs, detailed in \Cref{sec:method}.
}
\vspace{-3mm}
\end{figure*}

\subsection{Positional Encoding}\label{sec:pe}
In the field of sequential tasks, PE is crucial for models like Transformers to understand the sequential order of input tokens.
Absolute PE and relative PE are two prominent methods used to incorporate positional information into these models.
Absolute PE \citep{Vaswani2017AttentionIA} assigns fixed embeddings to each position in the input sequence using trigonometric functions like sine and cosine.
These embeddings are based solely on the position index and are not influenced by the token content, which are predefined and are generated as follows:
\begin{align}
\label{equ:float_pe}
\operatorname{PE}_{(pos, 2i)}=\sin \left(\frac{pos}{10000^{2i/d}}\right),\quad
\operatorname{PE}_{(pos, 2i+1)}=\cos \left(\frac{pos}{10000^{2i/d}}\right)
\end{align}
Here, $pos$ is the position and $d$ is the dimension.
In contrast, relative PE \citep{shaw-etal-2018-self,shiv2019novel,raffel2020exploring} captures the relationships between tokens by considering their relative distances.
This dynamic approach allows models to learn position-specific patterns and dependencies, which is beneficial for tasks requiring different sequence lengths or hierarchical structures.
\footnotetext[3]{Note that this is not a real relative PE. This term is adopted from the original papers.}

However, existing SNN architectures reveal a notable deficiency in the integration of an effective and biologically plausible PE mechanism.
As shown in \Cref{fig:intro}, current Transformer-based SNNs \citep{Zhou2022SpikformerWS,yao2023spike} are primarily tailored for image classification and predominantly rely on a convolutional layer to capture the relative positional information of image patches.
However, this approach resembles more of a spike-element-wise (SEW) residual connection \citep{Fang2021DeepRL} rather than a classic PE module, as it does not ensure that each image patch has a unique spike-form positional representation.
Furthermore, the addition between positional spikes and the original input spikes within these models may yield hardware-unfriendly non-binary integers (i.e., neither $0$ nor $1$), resulting from the addition of ``$1$'' and ``$1$''.
Additionally, our investigation reveals that even SNNs designed for sequential tasks, such as SpikeBERT \citep{Lv2023SpikeBERTAL,bal2024spikingbert}, SpikeGPT \citep{zhu2023spikegpt}, and SpikeTCN \citep{lv2024efficient}, also exhibit a notable absence of an effective spike-form PE mechanism for capturing positional information.

We think that an effective PE strategy should possess the following characteristics: \textbf{uniqueness of each position} and the \textbf{capacity to capture positional information from the input data}.
Furthermore, an optimal PE designed for SNNs should be \textbf{hardware-friendly} and \textbf{in spike-form}.

\subsection{Central Pattern Generators}
Central Pattern Generators (CPGs) are neural networks capable of producing rhythmic patterned outputs without sensory feedback \citep{marder2001central, grillner2006biological}.
These networks are fundamental for understanding motor control in vertebrates and invertebrates and are often applied to robotics and neural control systems.
Mathematically, CPGs can be modeled using systems of coupled nonlinear oscillators, and the general form can be written as:
\begin{align}
\dot{\mathbf{x}}  = \mathbf{F}(\mathbf{x}) + \mathbf{G}(\mathbf{x}, \mathbf{y}), \quad
\dot{\mathbf{y}}  = \mathbf{H}(\mathbf{y}) + \mathbf{K}(\mathbf{x}, \mathbf{y})
\end{align}
where \( \mathbf{x} \) and \( \mathbf{y} \) are the state variables (can be seen as membrane potential) of two coupled oscillators, \( \mathbf{F} \) and \( \mathbf{H} \) are intrinsic dynamics of the oscillators, and \( \mathbf{G} \) and \( \mathbf{K} \) are the coupling functions.


\section{Methods}\label{sec:method}
In biological systems, CPGs as well as other neurons do not transmit information directly through membrane potential but through spikes.
A burst of spikes will be generated only when the membrane potential of a CPG exceeds a certain threshold.
Therefore, we introduced the Heaviside step function in SNN, selecting only the part that exceeds the threshold, to design the CPG-PE.
In this section, we will first reveal the relationship between CPGs and PE.
Then we will introduce our proposed CPG-PE and its implementations.

\subsection{Relationship between Central Pattern Generators and Positional Encoding}
Consider one of the simplest CPGs with the following assumptions:
\begin{enumerate}
\item The CPG is a coupled nonlinear oscillator with 2 neurons whose states are represented as $\mathbf{x}(t)$ and $\mathbf{y}(t)$. 
\item Both neurons are autonomic neurons and will gain membrane voltage with constant speed, i.e., $\mathbf{F}(\mathbf{x})=b>0, \mathbf{H}(\mathbf{y})=d>0$.
\item Neuron represented by $\mathbf{x}$ will inhibits $\mathbf{y}$ while $\mathbf{y}$ excites $\mathbf{x}$. And the influence is proportional to the other neuron's state. Formally, $\mathbf{G}(\mathbf{x}, \mathbf{y})=a\mathbf{y}, \mathbf{K}(\mathbf{x}, \mathbf{y})=-c\mathbf{x}$ where $a > 0, c > 0$.
\end{enumerate}
Now the coupled oscillators can be represented as:
\begin{align}
    \dot{\mathbf{x}}(t) = a\mathbf{y}(t) + b, \quad
    \dot{\mathbf{y}}(t) = -c\mathbf{x}(t) + d
\end{align}

The general solution of this differential equation system is:
\begin{align}
\mathbf{x}(t)&=k_{1} \cos (\sqrt{ac}~t) + k_{2}\sqrt{\frac{a}{c}} \sin (\sqrt{ac}~t)+\frac{d}{c} \label{equ:solution_raw1}\\
\mathbf{y}(t)&=-k_{1} \sqrt{\frac{c}{a}}  \sin (\sqrt{ac}~t)+k_{2} \cos (\sqrt{ac}~t)-\frac{b}{a} \label{equ:solution_raw2}
\end{align}
where $k_{1}$ and $k_{2}$ are arbitrary constants. To simplify, we can further re-parameterize $t$ with $t'= t + \arctan(k_1 / ak_2)$ as is to choose another start point, then we can rewrite \Cref{equ:solution_raw1,equ:solution_raw2} as:
\begin{align}
\label{equ:x(t)}
    \mathbf{x}(t')=\sqrt{k_1^2+\frac{a}{c}k_2^2}\sin(\sqrt{ac}~t')+\frac{d}{c} =A_{1}\sin(w_{1}t')+b_1, \\
\label{equ:y(t)}
    \mathbf{y}(t')=\sqrt{\frac{c}{a}k_1^2+k_2^2}\cos(\sqrt{ac}~t')-\frac{b}{a}= A_{2}\cos(w_{2}t')+b_2
\end{align}

Comparing \Cref{equ:x(t),equ:y(t)} and \Cref{equ:float_pe}, we are astonished to find that \textbf{the PE in Transformers \citep{Vaswani2017AttentionIA} is a particular solution of the membrane potential variations in a specific type of CPG} with properly chosen $a, b, c, d$.
This finding suggests that the use of sinusoidal PE in Transformers is actually a bio-plausible choice that could possibly advance the model's ability to learn indexing and periodic information.

\begin{figure*}[htp]
\centering
\subfigure[]{
\centering
\includegraphics[width=0.47 \textwidth]{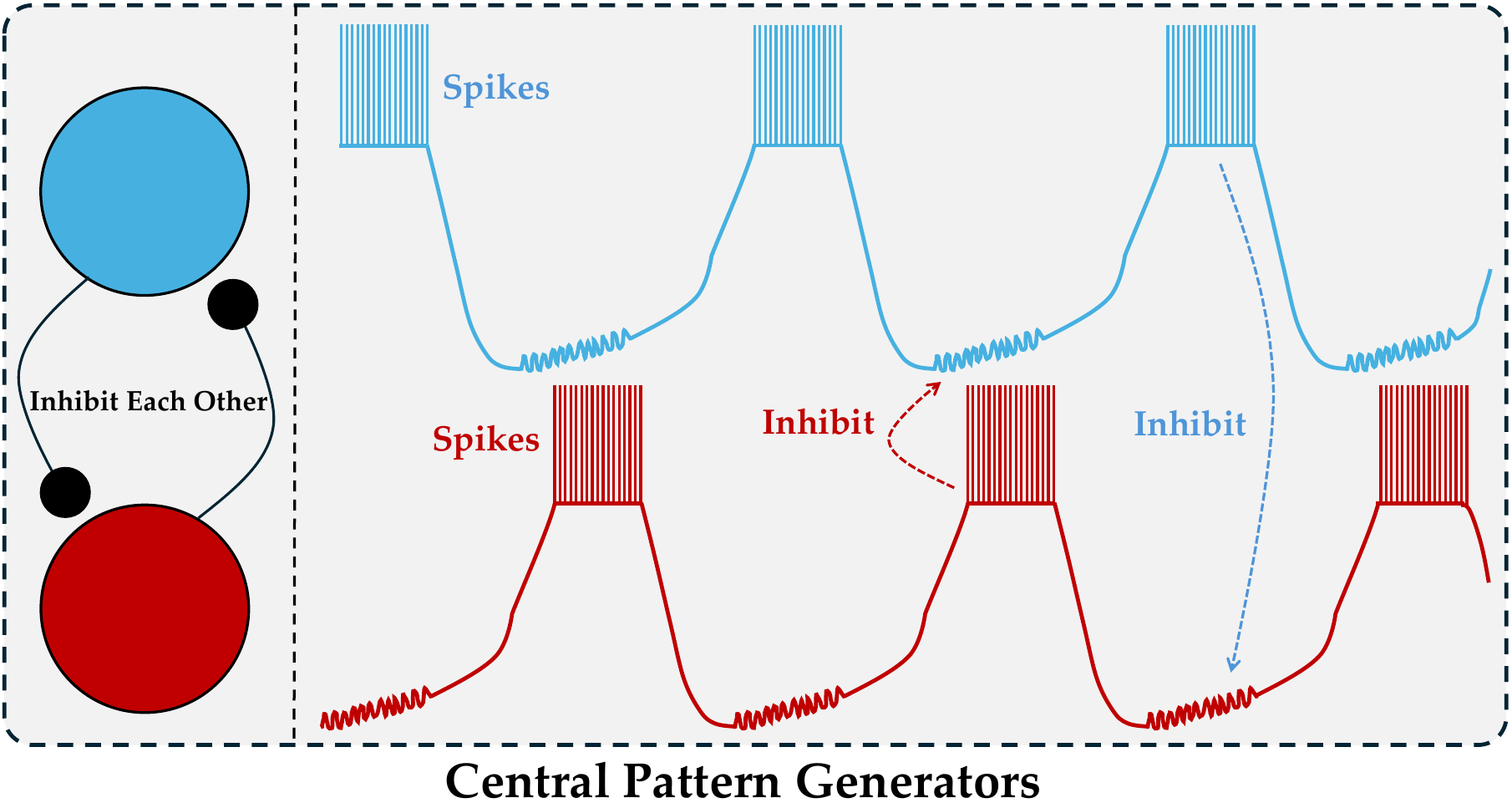}
}
\subfigure[]{
\centering
\includegraphics[width=0.49 \textwidth]{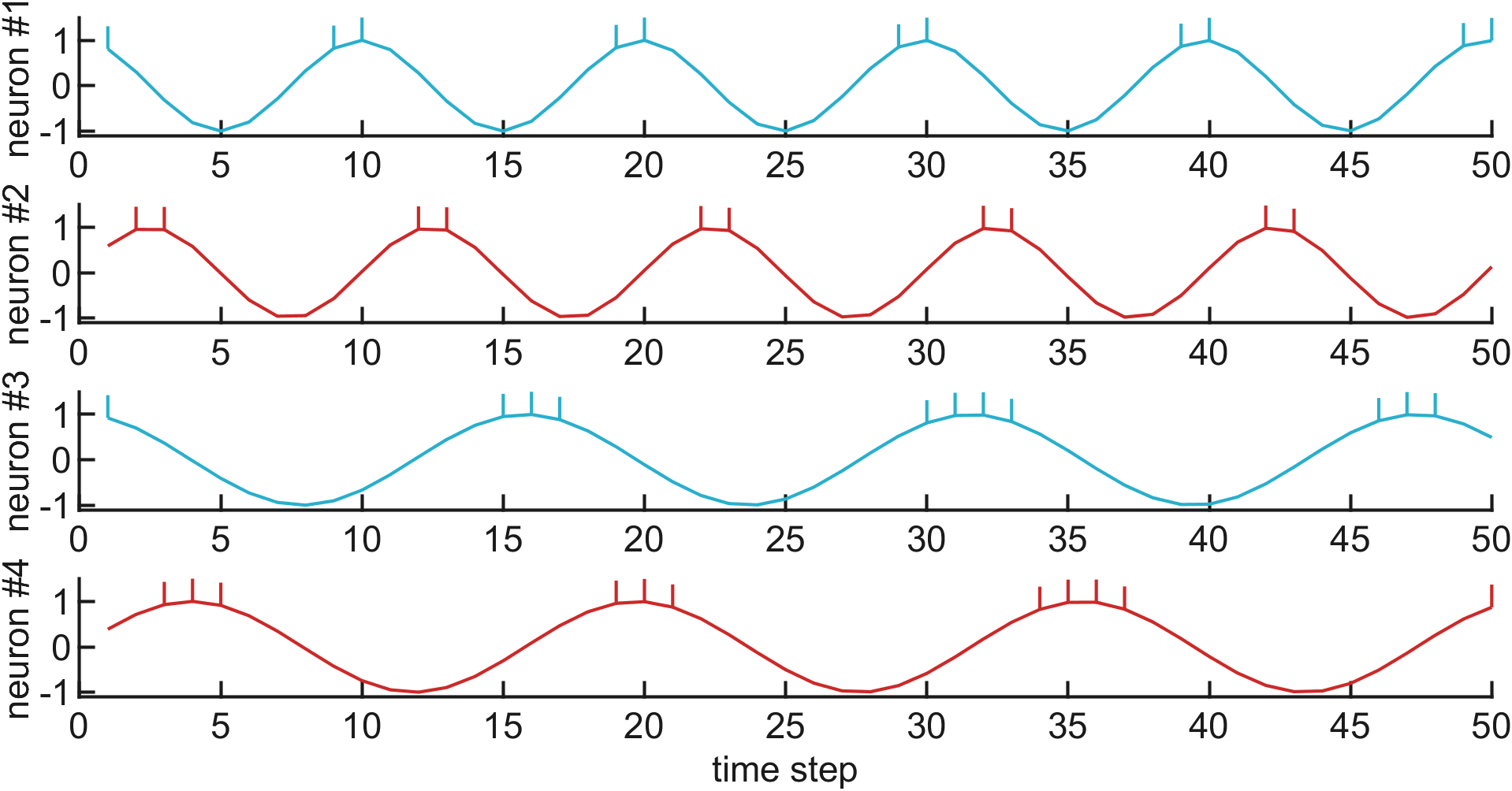}
}
\vspace{-3mm}
\caption{\label{fig:method_cpgpe}
(a) Illustration of a pair of CPG neurons demonstrating mutual inhibition through spiking activity. The spikes represent neural spikes that inhibit each other, exemplifying the coordination mechanism in CPG networks.
(b) Spike trains of the first $4$ CPG neurons. The curve represents the membrane potential, while the vertical lines represent spikes. 
}
\vspace{-3mm}
\end{figure*}

\subsection{CPG-based Positional Encoding}

Consider a system with $N$ pairs of CPG neurons, resulting in a total of $2N$ cells.
Then for $i=1, 2, ..., N$, the equations governing the CPG-PE are as follows:
\begin{align}
\label{equ:cpg_pe_sin}
\operatorname{CPG-PE}^{2i-1}(t)& =H\left(\cos\left(\eta\frac{t}{\tau^{\frac{i}{N}}}\right) - v^{\text{thres}} \right), \\
\label{equ:cpg_pe_cos}
\operatorname{CPG-PE}^{2i}(t)& =H\left(\sin\left(\eta\frac{t}{\tau^{\frac{i}{N}}}\right)- v^{\text{thres}} \right), 
\end{align}
where $\eta$ is a constant to control the period, $\tau$ represents the base period, and $v^{\text{thres}}$ denotes the membrane potential threshold.
Note that this threshold is different from the $U_{thr}$ of spike neurons described in \Cref{equ:s(t)}.
The Heaviside step function $H$ reflects a spike when the membrane potential exceeds the threshold.

It is important to clarify that the $t$ in \Cref{equ:cpg_pe_sin} and \ref{equ:cpg_pe_cos} is neither the time step in SNNs nor the position index.
Suppose the input spike matrix $X \in \{0,1\}^{T \times B \times L \times D}$, where $T$ is the time step in SNNs, $B$ is the batch size, $L$ is the sequence length of the input sample, $D$ is the feature dimension.
To ensure the uniqueness of each position at every time step, we flatten the dimensions $T$ and $L$ into a new dimension $T \times L$.
Therefore, $t$ ranges from $0$ to $T \times L$.
Notably, the entire CPG-PE operates in spike-form and is parameter-free.
To better understand CPG-PE, we draw a simple approximation of the resulting CPG spiking patterns under the assumption of a sequence length of \( L = 128 \) and \( N = 20 \) pairs of CPG neurons, illustrated in \Cref{fig:method_cpgpe} (b).

\subsection{Implementations}\label{sec:method_implem}

\begin{figure*}[htp]
\centering
\includegraphics[width=0.93 \textwidth]{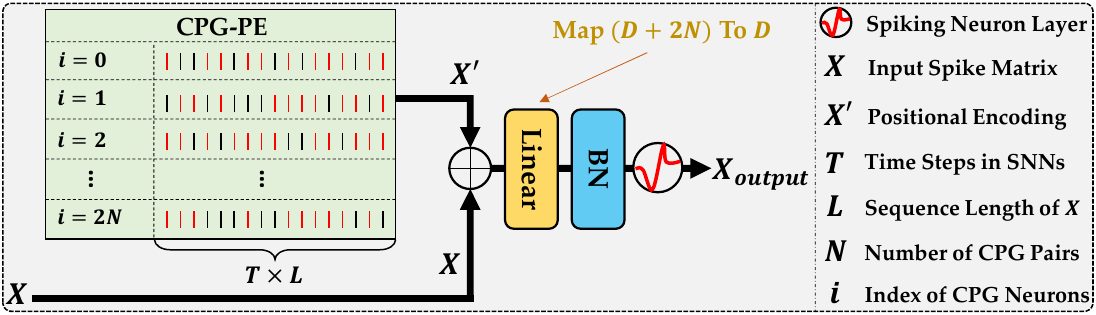}
\caption{\label{fig:implem}
Illustration of applying CPG-PE to SNNs. $X$, $X'$, and $X_{output}$ are all spike matrices.
}
\vspace{-3mm}
\end{figure*}

We design a simple implementation to apply CPG-PE to SNNs in a pluggable and hardware-friendly manner, shown in \Cref{fig:implem}.
Before diving into the details, we want to emphasize that the data transmitted in SNNs should always be in spike-form.
Therefore, the direct addition operation between two spike matrices, as used in \citep{Zhou2022SpikformerWS,yao2023spike}, should be forbidden.

Initially, CPG-PE encodes the positional information of the input spike matrix $X$, resulting in $X'$.
Then, to maintain binary values and avoid introducing non-binary elements, we opt to \textbf{concatenate} $X$ and $X'$ along the feature dimension.
Lastly, a linear layer is employed to map the feature dimension from $D+2N$ back to $D$, where $D$ is the feature dimension of $X$, and $N$ is the number of CPG pairs.
This effectively neutralizes the dimensional increase caused by concatenation.
The whole process can be formalized as follows:
\begin{align}
& X' = \operatorname{CPG-PE}(X), && X \in \{0, 1\}^{T \times B \times L \times D}, X' \in \{0, 1\}^{T \times B \times L \times 2N}\\
& X_{1} = X \oplus X', && X_{1} \in \{0, 1\}^{T \times B \times L \times (D+2N)}\\
& X_{output} = \mathcal{SN}\left(\operatorname{BN}\left(\operatorname{Linear}\left(X_{1}\right)\right)\right), && X_{output} \in \{0, 1\}^{T \times B \times L \times D}
\end{align}
where $\operatorname{BN}$ represents batch normalization and $\mathcal{SN}$ is a spike neuron layer.
Furthermore, CPG-PE necessitates that input samples be sequential data, making it directly applicable to time series data and natural language.
For image data, however, an adaptation is required:
images must be segmented into patches similar to the approach used in the Vision Transformer \citep{dosovitskiy2021an}.
Considering the compatibility with neuromorphic hardware, we also (1) implement CPG-PE with LIF neurons, and (2) integrate CPG-PE into a classic linear layer.
Please refer to \Cref{app:cpg_pe_with_lif,app:cpg_linear} for details.

\section{Experiments}

In this section, we conduct experiments to investigate the following research questions: \\
\textbf{RQ1}: Is our design of CPG-PE strategy effective and robust in sequential tasks? \\
\textbf{RQ2}: Can CPG-PE work well on image patches that have no inherent order? \\
\textbf{RQ3}: How will CPG's inner properties influence CPG-PE? \\
\textbf{RQ4}: Does our CPG-PE satisfy the requirements of a good PE tailored for SNNs? 

\subsection{Datasets}

To assess the PE capabilities of the compared models and answer \textbf{RQ1}, we conduct two sequential tasks: \textbf{time-series forecasting}, and \textbf{text classification}.
Following \citep{lv2024efficient}, we choose $4$ real-world datasets for time-series forecasting:
Metr-la \citep{li2017diffusion}: This dataset contains the average traffic speed data collected from the highways in Los Angeles County.
Pems-bay \citep{li2017diffusion}: It consists of average traffic speed data from the Bay Area.
Electricity \citep{lai2018modeling}: This dataset captures hourly electricity consumption measured in kilowatt-hours (kWh).
Solar \citep{lai2018modeling}: It includes data on solar power production.
For text classification, we follow \citep{Lv2023SpikeBERTAL} to conduct experiments on $6$ benchmarks including: Movie Reviews \cite{Pang2005SeeingSE}, SST-$2$ \cite{Socher2013RecursiveDM}, SST-$5$, Subj, ChnSenti, and Waimai.
In addition, to answer \textbf{RQ2}, we also conduct \textbf{image classification} experiments on $1$ static datasets CIFAR and $1$ neuromorphic datasets CIFAR10-DVS \citep{Li2017CIFAR10DVSAE}.
The dataset details and metrics are provided in \Cref{app:datasets}.

\subsection{Time-Series Forecasting}

As discussed in \Cref{sec:method_implem}, our proposed CPG-PE can be seamlessly integrated into any SNN capable of sequence processing.
Consequently, we applied CPG-PE to the SNN counterparts of Temporal Convolutional Networks (TCN) \citep{bai2018empirical}, Recurrent Neural Networks (RNN) to assess the efficacy of our method in enabling SNNs to capture positional information.
The results for TCN, SpikeTCN w/o PE, RNN, and Spike-RNN w/o PE are sourced from the previous study by \citep{lv2024efficient}.
In addition, we deliberately conducted experiments on PE in Spikformer to explore whether our specially designed CPG-PE is truly more suitable for SNNs than all previous PEs.
Notably, we also investigated the modularization of CPG, i.e., replacing all Linear layers with CPG-Linear layers (See \Cref{app:cpg_linear}), and its impact on the Spikformer model for time-series forecasting, i.e., Spikformer w/ CPG-Full.
We report the results on $4$ time-series forecasting benchmarks with various prediction lengths in \Cref{tab:tsf_table}.
We also list results from ANNs for reference.

\begin{table*}[]
\centering
\caption{
\label{tab:tsf_table}
Experimental results of time-series forecasting on $4$ benchmarks with various prediction lengths $6,24,48,96$.
``PE'' stands for positional encoding.
``w/o'' denotes ``without'' while ``w/'' denotes ``with''.
The best results of \textbf{SNNs} are formatted in \textbf{bold font format}.
$\uparrow$ ($\downarrow$) indicates the higher (lower) the better.
Shaded ones are ours.
All results are averaged across $3$ random seeds.
}
\resizebox{\linewidth}{!}{
\begin{tabular}{l:cc:r:cccc:cccc:cccc:cccc:c}
\toprule
\cline{1-21}
\multirow{2}{*}{Model} & \multirow{2}{*}{SNN} & \multirow{2}{*}{Spike PE} & \multirow{2}{*}{Metric} & \multicolumn{4}{c:}{\bf Metr-la} & \multicolumn{4}{c:}{\bf Pems-bay} & \multicolumn{4}{c:}{\bf Solar} & \multicolumn{4}{c:}{\bf Electricity} & \multirow{2}{*}{\textbf{Avg.}}\\
\cline{5-20}
& & & & $6$ & $24$ & $48$ & $96$ & $6$ & $24$ & $48$ & $96$& $6$ & $24$ & $48$ & $96$& $6$ & $24$ & $48$ & $96$ \\

\hline \hline

\multirow{2}{*}{TCN (ANN)}  & \multirow{2}{*}{\xmark}  & \multirow{2}{*}{--} & R$^2$$\uparrow$ & $.820$ & $.601$ & $.455$ & $.330$ & $.881$ & $.749$ & $.695$ & $.689$ & $.958$& $.871$ & $.737$ & $.661$ & $.975$  & $.973$ & $.968$ & $.962$ & $.770$\\
& & & RSE$\downarrow$ & $.446$ & $.665$ & $.778$ & $.851$ & $.373$ & $.541$ & $.583$ & $.587$ & $.210$& $.359$ & $.513$ & $.583$ & $.282$  & $.287$ & $.319$ & $.345$ & $.483$\\ 
\cline{1-21}

\multirow{2}{*}{SpikeTCN w/o PE \citep{lv2024efficient}} & \color{red}\multirow{2}{*}{\cmark} & \multirow{2}{*}{--} & R$^2$$\uparrow$ &  $.783$ & $\bm{.603}$ & $\bm{.468}$ & $.326$ & $.811$ & $.729$ & $.662$ & $.633$ & $.937$& $.840$ & $.708$ & $.650$ & $.970$  & $\bm{.963}$ & $.958$ & $.953$ & $.750$\\
& & & RSE$\downarrow$ & $.491$ & $.665$ & $\bm{.769}$ & $.865$ & $.469$ & $\bm{.541}$ & $.625$ & $.635$ & $.259$ & $.401$ & $.541$ & $\bm{.596}$ & $.333$  & $\bm{.342}$ & $\bm{.368}$ & $.389$ & $.518$\\ 
\cline{1-21}

 \rowcolor{cpgcolor} &  &  & R$^2$$\uparrow$ &  $\bm{.802}$ & $\bm{.603}$ & $.467$ & $\bm{.337}$ & $\bm{.839}$ & $\bm{.737}$ & $\bm{.684}$ & $\bm{.656}$ & $\bm{.951}$ & $\bm{.861}$ & $\bm{.729}$ & $\bm{.651}$ & $\bm{.974}$ & $.960$ & $\bm{.959}$ & $\bm{.956}$ & $\bm{.760}$\\
 \rowcolor{cpgcolor}\multirow{-2}{*}{SpikeTCN w/ CPG-PE} & \color{red}\multirow{-2}{*}{\cmark} & \color{red}\multirow{-2}{*}{\cmark} & RSE$\downarrow$ & $\bm{.469}$ & $\bm{.664}$ & $.770$ & $\bm{.859}$ & $\bm{.433}$ & $.555$ & $\bm{.604}$ & $\bm{.632}$ & $\bm{.222}$ & $\bm{.373}$ & $\bm{.521}$ & $.606$ & $\bm{.278}$ & $.380$ & $.374$ & $\bm{.370}$ & $\bm{.506}$ \\

\hline \hline

\multirow{2}{*}{RNN (ANN)} & \multirow{2}{*}{\xmark} & \multirow{2}{*}{\xmark} & R$^2$$\uparrow$ & $.844$ & $.600$ & $.442$ & $.307$ & $.870$ & $.775$ & $.690$ & $.683$ & $.959$ & $.830$ & $.810$ & $.718$ & $.978$ & $.972$ & $.971$ & $.964$ & $.776$ \\
& & & RSE$\downarrow$ & $.414$ & $.668$ & $.781$ & $.897$ & $.390$ & $.511$ & $.578$ & $.609$ & $.208$ & $.413$ & $.438$ & $.549$ & $.273$ & $.295$ & $.299$ & $.316$ & $.477$ \\
\cline{1-21}

\multirow{2}{*}{SpikeRNN w/o-PE  \citep{lv2024efficient}} & \color{red}\multirow{2}{*}{\cmark} & \multirow{2}{*}{--} & R$^2$$\uparrow$ & $\bm{.846}$ & $\bm{.622}$ & $.433$ & $.283$ & $.872$ & $.745$ & $.685$ & $.654$ & $.923$ & $.820$ & $\bm{.812}$ & $.714$ & $\bm{.977}$ & $\bm{.972}$ & $.962$ & $\bm{.960}$ & $.768$\\
& & & RSE$\downarrow$ & $\bm{.412}$ & $.648$ & $.794$ & $.935$ & $.387$ & $.528$ & $.588$ & $.634$ & $.278$& $.425$ & $\bm{.435}$ & $.586$ & $.267$  & $.296$ & $.346$ & $.481$ & $.503$\\
\cline{1-21}

\rowcolor{cpgcolor} & &  & R$^2$$\uparrow$ & $.844$ & $.621$ & $\bm{.438}$ & $\bm{.306}$ & $\bm{.874}$ & $\bm{.763}$ & $\bm{.688}$ & $\bm{.667}$ & $\bm{.934}$ & $\bm{.833}$ & $.811$ & $\bm{.724}$ & $\bm{.977}$ & $\bm{.972}$ & $\bm{.966}$ & $.958$ & $\bm{.773}$\\
\rowcolor{cpgcolor} \multirow{-2}{*}{SpikeRNN w/ CPG-PE}& \color{red}\multirow{-2}{*}{\cmark}  & \color{red}\multirow{-2}{*}{\cmark}& RSE$\downarrow$ & $.416$ & $\bm{.645}$ & $\bm{.782}$ & $\bm{.878}$ & $\bm{.380}$ & $\bm{.523}$ & $\bm{.579}$ & $\bm{.621}$ & $\bm{.264}$ & $\bm{.419}$ & $\bm{.435}$ & $\bm{.544}$ & $\bm{.265}$ & $\bm{.294}$ & $\bm{.315}$ & $\bm{.366}$ & $\bm{.482}$ \\
\hline \hline

\multirow{2}{*}{Transformer (ANN)} & \multirow{2}{*}{\xmark} & \multirow{2}{*}{\xmark} & R$^2$$\uparrow$ & $.727$ & $.554$ & $.413$ & $.284$ & $.785$ & $.734$ & $.688$ & $.673$ & $.953$ & $.858$ & $.759$ & $.718$ & $.978$ & $.975$ & $.972$ & $.964$ & $.752$\\
& & & RSE$\downarrow$ & $.551$ & $.704$ & $.808$ & $.895$ & $.502$ & $.558$ & $.610$ & $.618$ & $.223$ & $.377$ & $.504$ & $.545$ & $.260$ & $.277$ & $.347$ & $.425$ & $.512$ \\
\cline{1-21}




\multirow{2}{*}{Spikformer w/o PE} & \color{red}\multirow{2}{*}{\cmark} & \multirow{2}{*}{--} & R$^2$$\uparrow$ & $.697$ & $.491$ & $.383$ & $.242$ & $.768$ & $.684$ & $.678$ & $.663$ & $.903$ & $.819$ & $.715$ & $.656$ & $.956$ & $.955$ & $.953$ & $.943$ & $.719$\\
& & & RSE$\downarrow$ & $.581$ & $.753$ & $.828$ & $.917$ & $.521$ & $.607$ & $.613$ & $.627$ & $.319$ & $.439$ & $.548$ & $.602$ & $.371$ & $.375$ & $.386$ & $.450$ & $.559$\\
\cline{1-21}

\multirow{2}{*}{Spikformer w/ RPE \citep{Zhou2022SpikformerWS}} & \color{red}\multirow{2}{*}{\cmark} & \color{red}\multirow{2}{*}{\cmark} & R$^2$$\uparrow$ & $.713$ & $.527$ & $.399$ & $.267$ & $.773$ & $.697$ & $.686$ & $.667$ & $.929$ & $.828$ & $.744$ & $.674$ & $.959$ & $.955$ & $.955$ & $.954$ & $.733$\\
& & & RSE$\downarrow$ & $.565$ & $.725$ & $.818$ & $.903$ & $.514$ & $.594$ & $.606$ & $.621$ & $.272$ & $.426$ & $.519$ & $.586$ & $.373$ & $.371$ & $.379$ & $.382$ & $.541$\\
\cline{1-21}

 & \color{red}\multirow{2}{*}{\cmark}  & \multirow{2}{*}{\xmark} & R$^2$$\uparrow$ & $.699$ & $.502$ & $.409$ & $.255$ & $.762$ & $.704$ & $.687$ & $.666$ & $.934$ & $.834$ & $.752$ & $.699$ & $.970$ & $.967$ & $.960$ & $.957$ & $.734$\\
\multirow{-2}{*}{Spikformer w/ Float-PE}& & & RSE$\downarrow$ & $.578$ & $.744$ & $.810$ & $.912$ & $.527$ & $.588$ & $.605$ & $.623$ & $.264$ & $.418$ & $.512$ & $.563$ & $.307$ & $.322$ & $.356$ & $.362$ & $.531$ \\
\cline{1-21}

\rowcolor{cpgcolor} &  & & R$^2$$\uparrow$ &$\bm{.726}$ & $.526$ & $\bm{.419}$ & $\bm{.287}$ & $\bm{.780}$ & $.712$ & $\bm{.690}$ & $.666$ & $\bm{.937}$ & $.833$ & $\bm{.757}$ & $.707$ & $\bm{.972}$ & $.970$ & $.966$ & $.960$ & $\bm{.744}$\\
\rowcolor{cpgcolor}\multirow{-2}{*}{Spikformer w/ CPG-PE}& \color{red}\multirow{-2}{*}{\cmark} & \color{red}\multirow{-2}{*}{\cmark} & RSE$\downarrow$ & $\bm{.553}$ & $.720$ & $\bm{.806}$ & $\bm{.890}$ & $.508$ & $.580$ & $\bm{.602}$ & $.622$ & $\bm{.257}$ & $.420$ & $\bm{.506}$ & $.555$ & $\bm{.299}$ & $.310$ & $.314$ & $\bm{.355}$ & $\bm{.519}$\\
\cline{1-21}

\rowcolor{cpgcolor} &  &  & R$^2$$\uparrow$ & $.719$ & $\bm{.530}$ & $.417$ & $.286$ & $.779$ & $\bm{.714}$ & $.689$ & $\bm{.668}$ & $.936$ & $\bm{.835}$ & $\bm{.757}$ & $\bm{.709}$ & $.971$ & $\bm{.971}$ & $\bm{.968}$ & $\bm{.962}$ & $\bm{.744}$\\
\rowcolor{cpgcolor}\multirow{-2}{*}{Spikformer w/ CPG-Full}& \color{red}\multirow{-2}{*}{\cmark} & \color{red}\multirow{-2}{*}{\cmark} & RSE$\downarrow$ & $.560$ & $\bm{.719}$ & $.807$ & $.893$ & $\bm{.507}$ & $\bm{.577}$ & $.605$ & $\bm{.620}$ & $.260$ & $\bm{.417}$ & $.508$ & $\bm{.548}$ & $.304$ & $\bm{.308}$ & $\bm{.311}$ & $.439$ & $.523$ \\
\cline{1-21}
\bottomrule
\end{tabular}
}
\vspace{-3mm}
\end{table*}

In summary, the results presented in Table \ref{tab:tsf_table} indicate that SNNs equipped with the CPG-PE module significantly outperform their counterparts lacking the PE feature.
This finding effectively addresses \textbf{RQ1} from a time-series analysis perspective.
Detailed findings include:

\textbf{(1) CPG-PE enables SNNs to successfully capture positional information}. 
SNNs, including models such as Spike-TCN, Spike-RNN, and Spikformer, when integrated with CPG-PE, show superior performance compared to those without PE. Notably, CPG-PE also reduces the performance disparity between SNNs and traditional ANNs in time-series forecasting tasks, evidenced by an average increase of $0.013$ in R$^2$ and a decrease of $0.022$ in RSE.

\textbf{(2) CPG-PE is the most suitable position encoding strategy for Spikformer}.
In addition to CPG-PE, other encoding strategies such as Float-PE (the original PE in Transformer) and RPE (the original PE in Spikformer) were also evaluated.
The Spikformer equipped with CPG-PE emerged as the top-performing variant, confirming CPG-PE as the most suitable PE strategy for SNNs.

\textbf{(3) CPG-Full module can also effectively model the positional information of time series data}.
The CPG-Full module's performance in modeling positional information of time-series data is comparable to that of CPG-PE, with average R$^2$ values nearly identical to those of Spikformer with CPG-PE and significantly better than those of other models.

\subsection{Text Classification}
In addition to time-series forecasting, natural language processing (NLP) serves as another critical domain to assess the efficacy of the CPG-PE module in encoding positional information.
Following the pioneering work of \citep{Lv2023SpikeBERTAL}, who first employed Spikformer for text classification tasks, we extended this application to $6$ benchmark datasets.
We also include results from fine-tuned BERT for reference.

\begin{table} [ht]
\centering
\begin{center}
\caption{\label{tab:text_table}
Accuracy on $6$ text classification benchmarks.
The best results of SNNs and ANNs are formatted in bold font format.
Experimental results are averaged across $5$ random seeds.
}
\vspace{-3mm}
\resizebox{\linewidth}{!}{
\begin{tabular}{l:cc:c:cccc:cc:c} \hline \hline
\multirow{2}{*}{\textbf{Model}} & \multirow{2}{*}{\textbf{SNN}} & \multirow{2}{*}{\textbf{Spike PE}} & \multirow{2}{*}{\textbf{Param (M)}} & \multicolumn{4}{c:}{\bf English Dataset}  & 
\multicolumn{2}{c:}{\bf Chinese Dataset} & \multirow{2}{*}{\textbf{Avg.}}\\
\cline{5-10}
& & & & \textbf{MR} & \textbf{SST-2} & \textbf{Subj} & \textbf{SST-5} & \textbf{ChnSenti} & \textbf{Waimai}\\
\hline

Fine-tuned BERT \citep{Devlin2019BERTPO} & \xmark & \xmark & $109.8$ & $\bm{87.63}${\scriptsize $\pm 0.18$} & $\bm{92.31}${\scriptsize $\pm 0.17$} & $\bm{95.90}${\scriptsize $\pm 0.16$} & $\bm{50.41}${\scriptsize $\pm 0.13$} & $\bm{89.48}${\scriptsize $\pm 0.16$} & $\bm{90.27}${\scriptsize $\pm 0.13$} & $\bm{84.33}$\\
\hline 



Spikformer w/o PE \citep{Lv2023SpikeBERTAL} & {\color{red}\cmark} & -- & $109.8$ & $75.87${\scriptsize $\pm 0.35$} & $81.71${\scriptsize $\pm 0.31$} & $91.60${\scriptsize $\pm 0.30$} & $41.84${\scriptsize $\pm 0.39$} & $85.62${\scriptsize $\pm 0.25$} & $86.87${\scriptsize $\pm 0.28$} & $77.25$\\

Spikformer w/ Random-PE & {\color{red}\cmark} & {\color{red}\cmark} & $110.4$ & ${75.90}${\scriptsize $\pm 0.42$} & ${81.64}${\scriptsize $\pm 0.31$} & ${91.40}${\scriptsize $\pm 0.35$} & ${41.86}${\scriptsize $\pm 0.41$} & ${85.63}${\scriptsize $\pm 0.29$} & ${86.90}${\scriptsize $\pm 0.30$} & ${77.23}$\\

Spikformer w/ Float-PE & {\color{red}\cmark} & \xmark & $109.8$ & $79.67${\scriptsize $\pm 0.36$} & $82.18${\scriptsize $\pm 0.34$} & $92.20${\scriptsize $\pm 0.31$} & $42.58${\scriptsize $\pm 0.41$} & $85.71${\scriptsize $\pm 0.26$} & $88.34${\scriptsize $\pm 0.32$} & $78.44$\\
\rowcolor{cpgcolor}
Spikformer w/ CPG-PE [Ours] & {\color{red}\cmark} & {\color{red}\cmark} & $110.4$ & $\bm{82.42}${\scriptsize $\pm 0.42$} & $\bm{82.90}${\scriptsize $\pm 0.33$} & $\bm{92.50}${\scriptsize $\pm 0.25$} & $\bm{43.62}${\scriptsize $\pm 0.36$} & $\bm{86.54}${\scriptsize $\pm 0.26$} & $\bm{88.49}${\scriptsize $\pm 0.29$} & $\bm{79.41}$ \\
\hline


\hline\hline
\end{tabular}
}
\end{center}
\vspace{-3mm}
\end{table}

The results presented in \Cref{tab:text_table} shows that Spikformer enhanced with CPG-PE achieves the state-of-the-art performance across $6$ benchmarks, effectively addressing \textbf{RQ1}.
Meanwhile, we conducted a set of ablation experiments to eliminate the effects of increased parameter counts on model performance.
Specifically, we replaced the spike-form positional encoding matrix obtained from CPG with a randomly generated spike matrix (See ``Spikformer w/ Random PE'' Row).
By comparing these two configurations, we confirmed the effectiveness of our proposed CPG-PE.


\subsection{Image Classification}
In this section, we aim to answer \textbf{RQ2}.
To adapt the CPG-PE for image classification, it is essential to conceptualize the array of image patches as sequential data.
Consequently, some SNN models that do not incorporate a concept of ``sequence length'' in their spike matrices, such as SEW-Resnet \citep{Fang2021DeepRL}, are incompatible with the integration of a CPG-PE module.
Therefore, we only consider ViT-liked SNN, i.e. Spikformer, in this experiment.
We also include results from ViTs for reference.

\begin{table}[ht]
\caption{
\label{tab:image_classification}
Evaluation on image classification benchmarks.
Float-PE denotes the original PE of the Transformer, while RPE denotes the original PE of the Spikformer.
Numbers with $^*$ denote our implementation.
The best results of SNNs and ANNs are formatted in bold font format.
All results are averaged across $4$ random seeds.
}
\centering
\vspace{-3mm}
\resizebox{\linewidth}{!}{
\begin{tabular}{l:cc:cc:cc:cc:c}
\toprule
\multirow{2}{*}{\textbf{Model}} & \multirow{2}{*}{\textbf{SNN}} & \multirow{2}{*}{\textbf{Spike PE}} & \multicolumn{2}{c:}{\bf CIFAR10}  & \multicolumn{2}{c:}{\bf CIFAR10-DVS}  & \multicolumn{2}{c:}{\bf CIFAR100}  & \multirow{2}{*}{\textbf{Avg.}} \\
\cline{4-9}
& & & Param (M) & Accuracy & Param (M) & Accuracy & Param (M) & Accuracy & \\
\hline

Vision-Transformer \citep{dosovitskiy2021an} & \xmark & \xmark & $9.32$ & $\bm{96.73}$ & -- & -- & $9.36$ & $\bm{81.02}$ & --\\

\hline

Spikformer w/o PE & {\color{red} \cmark} & -- & $8.00$ & $93.77$ & $1.99$ & $76.40$ & $8.04$ & $73.59$ & $81.25$\\

Spikformer w/ Random-PE & {\color{red} \cmark} & {\color{red} \cmark} & $8.17$ &  $93.85$ & $2.06$ & $76.44$ & $8.20$ & $73.54$ & $81.27$\\

Spikformer w/ Float-PE & {\color{red} \cmark} & \xmark & $8.00$ & $94.42$ & $1.99$ & $77.60$ & $8.04$ & $74.73$ & $82.25$\\

Spikformer w/ RPE \citep{Zhou2022SpikformerWS} & {\color{red} \cmark} & {\color{red} \cmark} & $9.33$ & $\;\,94.64^*$ & $2.57$ & $\;\,77.95^*$ & $9.37$ & $\;\,76.78^*$ & $83.12$\\
\rowcolor{cpgcolor}
Spikformer w/ CPG-PE [Ours] & {\color{red} \cmark} & {\color{red} \cmark}  & $8.17$ & $\bm{94.82}$ & $2.06$ & $\bm{78.06}$ & $8.20$ & $\bm{77.27}$ & $\bm{83.38}$\\
\rowcolor{cpgcolor}
Spikformer w/ CPG-PE [Equal Param] & {\color{red} \cmark} & {\color{red} \cmark}  & $7.99$ & $94.60$ & $1.99$ & $78.00$ & $8.02$ & $76.91$ & $83.17$\\




\bottomrule
\end{tabular}
}
\end{table}

We report the parameter counts and classification accuracy in \Cref{tab:image_classification}.
To elaborate, Spikformer with CPG-PE outperforms other variants, demonstrating the effectiveness of CPG-PE even when the sequence is an array of image patches lacking inherent order.
Notably, owing to our streamlined implementation, the parameter count for Spikformer with CPG-PE is significantly reduced compared to the original Spikformer w/ RPE \citep{Zhou2022SpikformerWS}, with a reduction of $1.16$ M.
What's more, we conducted ablation experiments on model parameters by reducing the parameter count of Spikformer with CPG-PE to be comparable to Spikformer w/o PE, allowing for a more direct performance comparison, as shown in the last line in \Cref{tab:image_classification}.
The results on ImageNet are reported in \Cref{app:imagenet}.

However, it is essential to acknowledge that the improvements in image classification are relatively modest compared to those observed in time series and text applications.
This phenomenon can largely be attributed to the intrinsic \emph{non-ordered nature} of image patches.
Unlike text or time series data, where sequential order is crucial and inherently informative, image patches do not possess a natural or fixed sequence.
This lack of order means that traditional methods of positional encoding, which significantly benefit ordered data by providing contextual positioning, are less effective.
Thus, the application of our positional encoding techniques, optimized for data with inherent sequential order, does not translate as effectively to the domain of image classification. 

\subsection{Sweeping CPG properties}
\label{sec:sweep_cpg}

In this section, we investigate the influence of CPG properties on the ability to model positional information, addressing \textbf{RQ3}. To this end, we evaluated the Spikformer model with CPG-PE by varying the base period $\tau$ and the number of CPG pairs $N$ (see \Cref{equ:cpg_pe_cos,equ:cpg_pe_sin}) in time-series forecasting and image classification tasks.

\begin{figure*}[htp]
\centering
\subfigure[]{
   \centering
   \includegraphics[width=0.23 \textwidth]{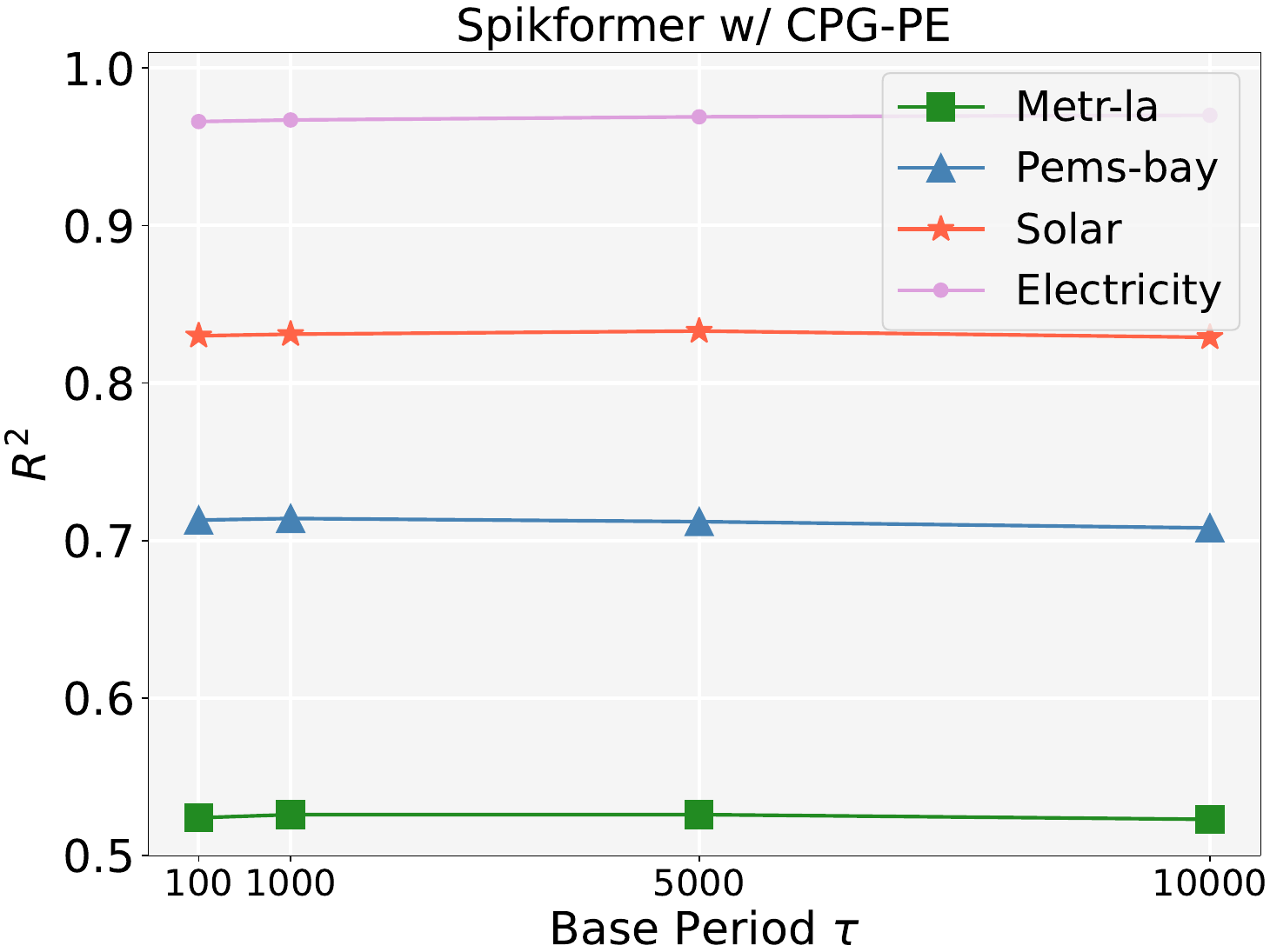}
}
\subfigure[]{
    \centering
    \includegraphics[width=0.23 \textwidth]{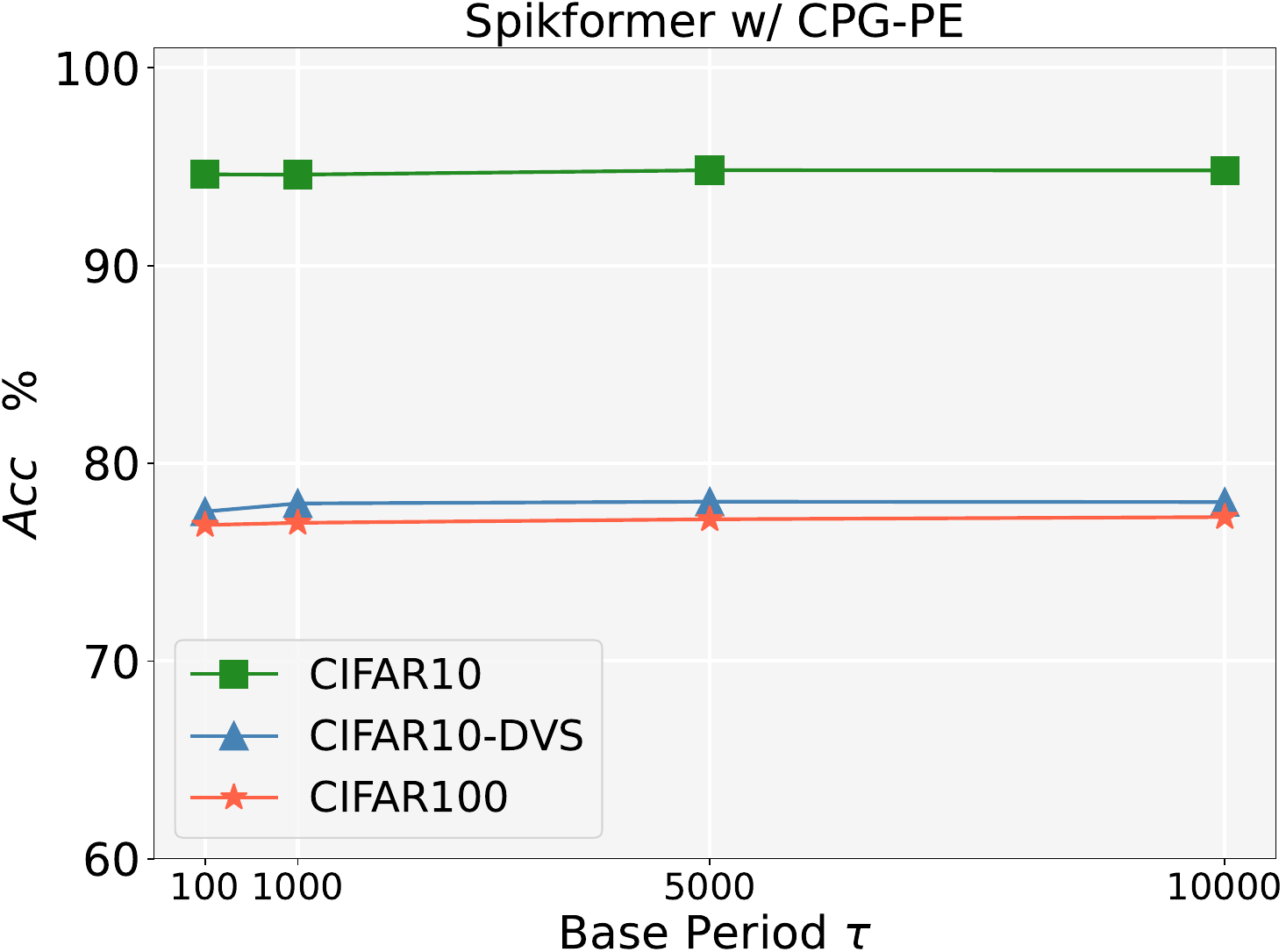}
}
\subfigure[]{
    \centering
    \includegraphics[width=0.23 \textwidth]{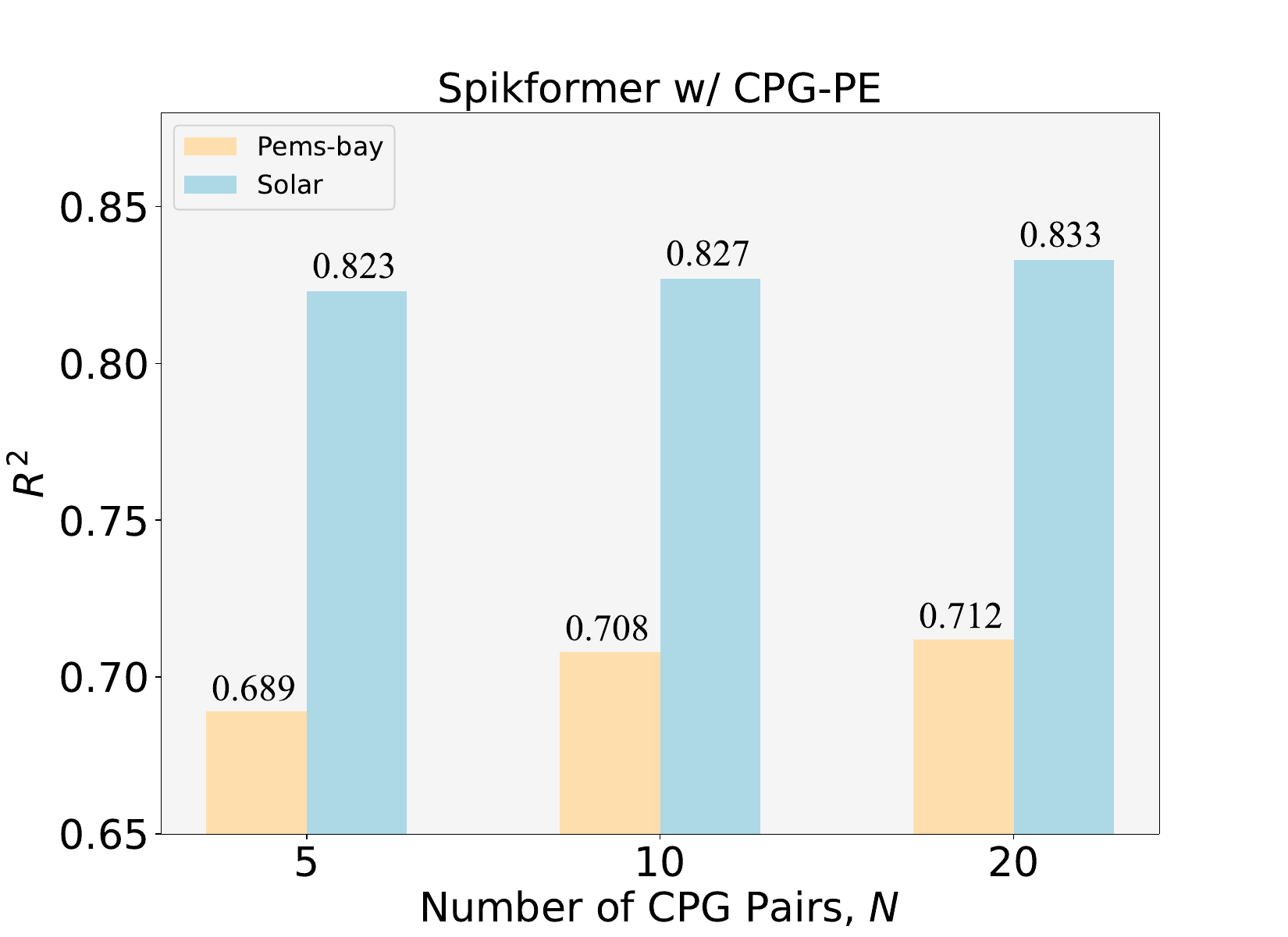}
}
\subfigure[]{
    \centering
    \includegraphics[width=0.23 \textwidth]{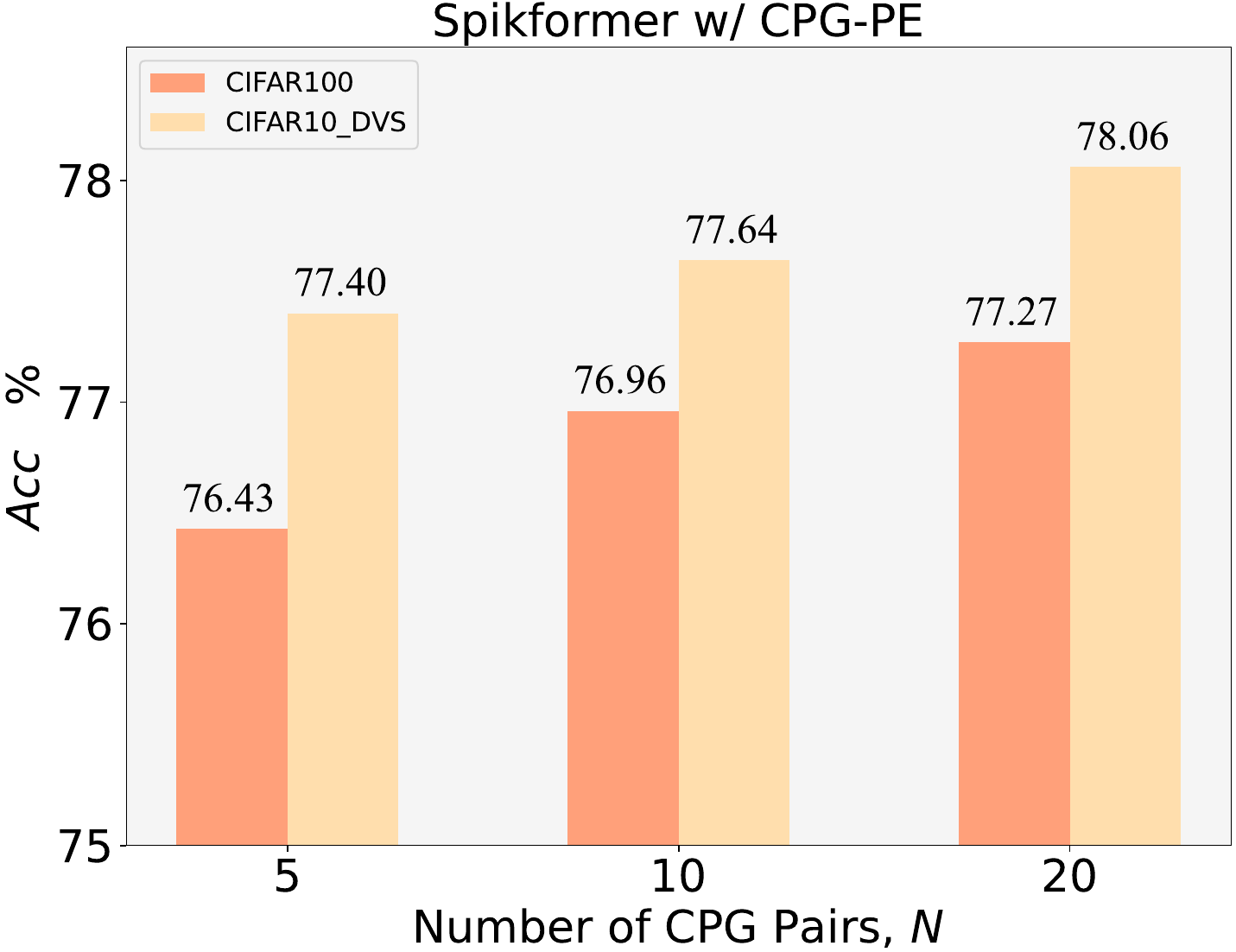}
}
\vspace{-3mm}
\caption{\label{fig:cpg_property}
(a)(c) $R^2$ versus $\tau$ and $N$ on time-series forecasting tasks.
(b)(d) Accuracy versus $\tau$ and $N$ on image classification tasks.
$\tau \in \{100,1000,5000,10000\}$, $N \in \{5,10,20\}$.
}
\vspace{-1mm}
\end{figure*}

From \Cref{fig:cpg_property} (a) and (b), we observe that CPG-PE is insensitive to the base period $\tau$ (in biological neurons, $\tau$ is affected by the physiological properties of the CPG circuit such as RC constant and synaptic delay). 
The sequence lengths ($T \times L$) of the time series and image patches are no larger than $672$ $(4 \times 168)$ for all benchmarks, preventing repetitions in CPGs.
Therefore, when $N=20$, sweeping $\tau \in \{100, 1000, 5000, 10000\}$ makes minor influence on performance.
Furthermore, \Cref{fig:cpg_property} (c) and (d) demonstrate that when $\tau=10000$, increasing the number of CPG pairs $N$ enhances Spikformer's performance.
This is reasonable because more CPG neurons reduce repetitions in positional representations of $X'$.

\subsection{Positional Encoding Analysis}
In this section, we want to address \textbf{RQ4}.
As mentioned in \Cref{sec:pe}, an ideal PE method for SNNs should include the following characteristics:
(1) \textbf{Uniqueness of each position};
(2) \textbf{Ability to discern positional information};
(3) \textbf{Compatibility with neuromorphic hardware};
(4) \textbf{Formulation in spike-form}.
Our experiments confirm the ability to discern positional information (2), our implementations ensure compatibility with neuromorphic hardware (3), and the CPG-PE is inherently formulated in spike-form, satisfying (4).
Therefore,  in order to assess (1), the uniqueness of each position, we would like to compare the CNN-based RPE in \citep{Zhou2022SpikformerWS,yao2023spike} and our proposed CPG-PE, focusing specifically on their capacity to provide distinct positional signals.
This analysis was conducted using the CIFAR10-DVS dataset, where we calculated the repetition rate of spike positional representations across all positions.
Our findings were notable: the positional spike matrices produced by RPE exhibited a repetition rate as high as $\bm{12.19\%}$, which significantly undermines its effectiveness for PE.
In contrast, our proposed CPG-PE exhibited no repetition, demonstrating that our CPG-PE is well-suited for serving as the PE module in SNNs.
Please refer to \Cref{app:imple} for details.

\section{Related Work}

\subsection{Spike Encoding Methods}

Spiking neural networks employ several coding methods to encode input information, each offering unique advantages.
Direct coding \citep{yao2023spike,zhou2024spikformer}, the simplest form and widely-used in image tasks, directly associates spikes with specific values or events, providing straightforward and interpretable outputs but often lacking efficiency for complex tasks.
Rate coding \citep{lv2023spiking,kim2022rate}, where the input is represented by the frequency of spikes within a given timeframe, is more robust and widely used but can be less precise due to its reliance on averaged spike rates.
Temporal coding (a.k.a latency coding) \citep{han2020deep,comsa2022temporal} encodes information based on the timing of individual spikes, allowing for high temporal precision and efficient representation of dynamic inputs, though it can be computationally demanding.
In addition, delta coding \citep{yoon2016lif} represents changes in input signals through spikes, focusing on differences rather than absolute values, which can enhance efficiency and response times but may introduce complexity in decoding.
Each of these methods contributes to the versatility and applicability of SNNs in various domains, from neuroscience to artificial intelligence.
The SNNs we considered in this paper should fall into the category of rate coding since back-prop is conducted on spike rate.
Meanwhile, CPG-PE can be considered converting temporal information into spike rate of a group of neurons (Equations \ref{equ:cpg_pe_sin} and \ref{equ:cpg_pe_cos}), and this is why CPG-PE can improve performance for sequential data.
It is possible to introducing learning algorithms of temporal coding for the CPG neurons to tackle more complex sequence structure, which remains as future work.

\subsection{Positional Encoding in SNNs}
Currently, few works have demonstrated the importance of PE approaches in SNNs.
Spikformer \citep{Zhou2022SpikformerWS} and Spike-driven Transformer \citep{yao2023spike} utilize a combination of ``one convolutional layer + one batch normalization layer + one spiking neuron layer '' to generate learnable ``relative positional encoding''.
From our perspective, this strategy is more like a spike-element-wise residual connection \citep{Fang2021DeepRL}, rather than a classic PE module.
The unique representation of each position is a fundamental requirement for a robust PE module.
However, the spike position matrix generated by their method may result in the same spike representation for different positions.
Additionally, the addition of the input spike matrix and the position spike matrix will result in the occurrence of non-binary numbers (i.e., $2$) due to the addition of $1$ and $1$.
For spiking graph neural networks, \citep{zhao2024dynamic} proposed learnable positional graph spikes, aiming to capture neighbor information within graphs rather than sequences.
Therefore, drawing inspiration from the periodic automatic spike generation pattern of CPGs, we propose a biologically plausible and effective spike-form absolute position encoding method called CPG-PE.

\section{Rethinking the Role of CPGs in Neuroscience}

Our study also provides novel insights into neuroscience on understanding the role of CPG in nervous systems.
While traditionally CPG is believed to play a crucial role in producing the rhythmic motor patterns necessary for locomotion and other repetitive movements \citep{marder2001central, grillner2006biological}, the analogy to PE in this work reveals that CPG can make a significant contribution in processing sequential data by encoding the positional information into unique spiking patterns at different times.
This does not only work for time-series sensory input like auditory signals but also for visual sensory data: e.g. when a person looks at an image, saccades (eye movements) allow retinal neurons to receive different parts of the image at different times.
This indicates that CPG neurons could potentially be utilized to encode positional information.
Another extensive thought is that as PE can be learnable in ANNs, CPG may also benefit from adaptability to the data \citep{yuste2005cortex}.
The hypothesis, however, remains to be examined through neuroscientific experiments \citep{marblestone2016toward}.




\section{Conclusion}

In conclusion, inspired by central pattern generators, we introduce a pioneering position encoding approach termed CPG-PE, specifically tailored to mitigate the constraints associated with current PE techniques within SNNs.
We mathematically prove that abstract PE in the Transformer is a particular solution of the membrane potential variations in a specific type of CPG.
Furthermore, through comprehensive empirical investigations across diverse domains including time-series forecasting, natural language processing, and image classification, we demonstrate that the CPG-PE satisfies all the requirements of PE tailored for SNNs.
The limitations and future work are discussed in \Cref{app:lim}.

\bibliographystyle{unsrt}
\bibliography{neurips_2024}

\newpage

\appendix

\section*{Acknowledgments}
The authors would like to thank the anonymous reviewers for their valuable comments.
This work was supported partially by National Natural Science Foundation of China (No. 62076068).

\section*{Broader Impact}
\label{sec:broader_impact}
This work aims to advance the field of spiking neural networks (SNNs).
Unlike artificial neural networks (ANNs) which have been applied widely in people's lives, SNNs are still undergoing fundamental research.
We do not see any negative societal impacts of this work.

\section*{Reproducibility Statement}
The authors have diligently worked to ensure the reproducibility of the empirical results presented in this paper. The datasets, experimental setups, evaluation metrics, and hyperparameters are thoroughly described in \Cref{app:datasets,app:imple}.
Furthermore, the source code for the proposed PE method has been available at \url{https://github.com/microsoft/SeqSNN}.

\section{Datasets} \label{app:datasets}

\subsection{Time-series Forecasting}

Detailed statistical characteristics and distribution ratios for each dataset are provided in the following:

\begin{table}[h]
\centering
\caption{The statistics of time-series datasets.}
\label{tab:exp setting}
\resizebox{0.8\linewidth}{!}{
\begin{tabular}{lcccc}
\hline
Dataset & Samples & Variables & Observation Length & Train-Valid-Test Ratio \\
\hline
Metr-la & $34,272$ & $207$ & $12, (\operatorname{short-term})$ & $(0.7, 0.2, 0.1)$ \\
Pems-bay & $52,116$ & $325$ & $12, (\operatorname{short-term})$  & $(0.7, 0.2, 0.1)$ \\
Solar-energy & $52,560$ & $137$ & $168, (\operatorname{long-term})$ & $(0.6, 0.2, 0.2)$  \\
Electricity & $26,304$ & $321$ & $168, (\operatorname{long-term})$ & $(0.6, 0.2, 0.2)$ \\
\hline
\end{tabular}
}
\end{table}

\subsection{Text Classification}
Here are the datasets we used in text classification experiments:
\begin{itemize}
\setlength{\itemsep}{0pt}
\setlength{\parsep}{0pt}
\setlength{\parskip}{0pt}

\item{\textbf{MR}}:
MR, which stands for Movie Review, is a dataset containing movie-review documents labeled based on their overall sentiment polarity (positive or negative) or subjective rating \citep{Pang2005SeeingSE}.

\item{\textbf{SST-$\bf 5$}}:
SST-$5$ includes $11,855$ sentences from movie reviews for sentiment classification across $5$ categories: very negative, negative, neutral, positive, and very positive \citep{Socher2013RecursiveDM}.

\item{\textbf{SST-$\bf 2$}}:
SST-$2$ is the binary version of SST-$5$, containing only $2$ classes: positive and negative.

\item{\textbf{Subj}}:
The Subj dataset is designed to classify sentences as either subjective or objective\footnote{\scriptsize{\url{https://www.cs.cornell.edu/people/pabo/movie-review-data/}}}.

\item{\textbf{ChnSenti}}:
ChnSenti consists of approximately $7,000$ Chinese hotel reviews, each annotated with a positive or negative label\footnote{\scriptsize{\url{https://raw.githubusercontent.com/SophonPlus/ChineseNlpCorpus/master/datasets/ChnSentiCorp_htl_all/ChnSentiCorp_htl_all.csv}}}.

\item{\textbf{Waimai}}:
This dataset contains around $12,000$ Chinese user reviews from a food delivery platform, intended for binary sentiment classification (positive and negative)\footnote{\scriptsize{\url{https://raw.githubusercontent.com/SophonPlus/ChineseNlpCorpus/master/datasets/waimai_10k/waimai_10k.csv}}}.
\end{itemize}

\subsection{Image Classification}
Here are the datasets we used in image classification experiments:
CIFAR dataset comprises a collection of $60,000$ images, partitioned into $50,000$ training and $10,000$ testing images, each with a resolution of $32\times32$ pixels.
The CIFAR10-DVS dataset represents a neuromorphic adaptation of this original set, where static images have been transformed to accommodate the recording capabilities of a Dynamic Vision Sensor (DVS) camera.
This conversion results in a dataset consisting of $9,000$ training samples and $1,000$ test samples with $128\times128$ resolution.

\section{Experiment Settings}\label{app:imple}

\subsection{Time-series Forecasting}

\paragraph{Metrices}
The metrics we used in time-series forecasting are the coefficient of determination (R$^2$) and the Root Relative Squared Error (RSE).
\begin{align}
R^2&=\frac1{MCL}\sum_{m=1}^M\sum_{c=1}^C\sum_{l=1}^L\left[1-\frac{(Y^m_{c,l}-\hat{Y}^m_{c,l})^2}{(Y^m_{c,l}-\bar{Y}_{c,l})^2}\right].\\
\mathrm{RSE}&=\sqrt{\frac{\sum_{m=1}^M||\mathbf{Y}^m-\hat{\mathbf{Y}}^m||^2}{\sum_{m=1}^M||\mathbf{Y}^m-\bar{\mathbf{Y}}||^2}}
\end{align}
In these formulas, $M$ symbolizes the size of the test sets, $C$ denotes the number of channels, and $L$ signifies the length of predictions. $\bar{\mathbf{Y}}$ represents the average of $\mathbf{Y}^m$. The term $Y^m_{c,l}$ refers to the $l$-th future value of the $c$-th variable for the $m$-th sample, while $\bar{Y}_{c,l}$ indicates the mean of $Y^m_{c,l}$ across all samples. The symbols $\hat{\mathbf{Y}}^m$ and $\hat{Y}_{c, l}^{m}$ are used to represent the ground truth values.
Compared to Mean Squared Error (MSE) or Mean Absolute Error (MAE), these metrics offer greater resilience against the absolute values of the datasets, making them particularly useful in the time-series forecasting setting.

\paragraph{Model Architecture}
All SNNs take $4$ time steps.
For SpikeTCNs and SpikeRNNs, we follow the same settings as \citep{lv2024efficient}.
We construct all Spikformer as $2$ blocks, setting the feature dimension as $256$, and the hidden feature dimension in FFN as $1024$.
For CPG-PE settings, we set $\tau = 10000.0$, $N=20$, $\eta=1$, and $v^{\text{thres}}=0.8$.

\paragraph{Training Hyper-parameters}
we set the training batch size as $64$ and adopt Adam \cite{kingma2014adam} optimizer with a cosine scheduler of learning rate $1\times 10^{-4}$.
An early stopping strategy with a tolerance of $30$ epochs is adopted.
We conducted time-series forecasting experiments on 24G-V100 GPUs.
On average, a single experiment takes about $1$ hour under the settings above.

\subsection{Text Classification}
\paragraph{Model Achirecture}
All models are with $12$ encoder blocks and $768$
feature embedding dimension.
It is important to note that the original implementation of \citep{Lv2023SpikeBERTAL} incorporates a layer normalization module that poses challenges to hardware compatibility.
To address this, we have substituted layer normalization with batch normalization in our directly-trained Spikformer models for text classification tasks.
For CPG-PE settings, we set $\tau = 10000.0$, $N=20$, $\eta=1$, and $v^{\text{thres}}=0.8$.

\paragraph{Training Hyper-parameters}
We directly trained Spikformers with arctangent surrogate gradients on all datasets.
We use the BERT-Tokenizer in Huggingface\footnote{\url{https://huggingface.co/}} to tokenize the sentences to token sequences.
We pad all samples to the same sequence length of $256$.
We conducted text classification experiments on $4$ RTX-3090 GPUs, and set the batch size as $32$, optimizer as AdamW \cite{loshchilov2018decoupled} with weight decay of $5 \times 10^{-3}$, and set a cosine scheduler of starting learning rate of $5\times 10^{-4}$.
What's more, in order to speed up the training stage, we adopt the automatic mixed precision training strategy.
On average, a single experiment takes about $1.5$ hours under the settings above.

\subsection{Image Classification}

\paragraph{Model Architecture}
For all Spikformer models, we standardized the configuration to include $4$ time steps. Specifically, for the CIFAR10 and CIFAR100 datasets, the models were uniformized with $4$ encoder blocks and a feature embedding dimension of $384$.
For the CIFAR10-DVS dataset, the models were adjusted to have $2$ encoder blocks and a feature embedding dimension of $256$.
For CPG-PE settings, we set $\tau = 10000.0$, $N=20$, $\eta=2\pi$, and $v^{\text{thres}}=0.8$.

\paragraph{Training Hyper-parameters}
We honestly follow the experimental settings in \citep{Zhou2022SpikformerWS}, whose source code and configuration files are available at \url{https://github.com/ZK-Zhou/spikformer}.
As the training epochs are quite big ($300$ epochs) in their settings, we choose to use one 80G-A100 GPU, and it takes about $3$ hours to conduct a single experiment, on average.

\subsection{Details about Positional Encoding Analysis}

We conducted positional encoding analysis experiments on the CIFAR10-DVS dataset.
For the original Spikformer with relative positional encoding (RPE) as described by \citep{Zhou2022SpikformerWS}, the input and output channels of Conv2d are both set to $384$.
In our Spikformer with CPG-PE, the parameters are set to \(\tau = 10000.0\) and \(N = 20\).
Given the time step \(T = 4\) and the sequence length \(L = 160\) for the image patches in CIFAR10-DVS samples, the total "length" \(T \times L\) in CPG-PE is $640$.
We then calculated the repetition rate of positions. The results showed that the repetition rate for RPE is $12.19\%$, whereas for CPG-PE, it is $0.00\%$.

\section{Implement CPG-PE with LIF Neurons}\label{app:cpg_pe_with_lif}
In this section, we demonstrate that CPG-PE is a hardware-friendly design. While implementing the sinusoidal potential on the neuromorphic chips is not challenging (e.g., by maintaining additional LC circuits), we show how a CPG-PE neuron can be physically implemented with only 2 LIF neurons defined by \Cref{equ:membranePotential,equ:ht,equ:s(t)} and thus introducing no extra efforts on chip designs. 

A CPG-PE neuron, after discretization, can be viewed as an autonomic neuron that will emit a burst of $K$ spikes after resting for $R$ time steps. The key idea is to set two LIF neurons, namely the \textit{Emitter} and the \textit{Resetter}. The emitter will draw constant current from the source, and as soon as its membrane potential reaches the threshold after $R$ time steps, it will start emitting spikes constantly until receiving the reset signal from the resetter. The resetter, which will remain at the resting potential until it receives signals from the emitter, will count the number of spikes and emit a reset signal (inhibition signal) to the emitter after receiving $K$ spikes.

We first prove the following Lemma, which establishes the relationship between the start time of the first spike and a constant input current. 
\begin{lemma}
\label{lem:LIF_spike}
    Given an LIF neuron defined by \Cref{equ:membranePotential,equ:ht,equ:s(t)} with decay rate $\beta$ and threshold $U_{thr}$, starting with resting potential $U(0)=0$, if fed with the constant current $I(t)=I_c > 0$, the first spike will emit at:
    \begin{align}
        T_{min} = \left\lceil \log_\beta(\beta-\frac{U_{thr}\beta(1-\beta)}{I_c})\right\rceil.
    \end{align}
\end{lemma}
\begin{proof}
    By definition, before the time to emit the first spike, we have $S(t)=0$. Thus \Cref{equ:membranePotential} can be rewrite as:
    \begin{align}
        U(k\Delta t) = \beta U((k-1)\Delta t)+I_c.
    \end{align}
    Simplifying the recurrence relation, we can obtain:
    \begin{align}
        U(k\Delta t) = \frac{I_c}{\beta}\left(\frac{1-\beta^k}{1-\beta}-1\right).
    \end{align}
    The first spike is generated when $U(k\Delta t)\leq U_{thr}$, thus we have:
    \begin{align}
    k \geq \log_\beta(\beta-\frac{U_{thr}\beta(1-\beta)}{I_c}),
    \end{align}
    that is to say,
    \begin{align}
    T_{min} = \left\lceil \log_\beta(\beta-\frac{U_{thr}\beta(1-\beta)}{I_c})\right\rceil. 
    \end{align}
\end{proof}

Now we can implement the CPG-PE with LIF neurons:
\begin{theorem}
\label{thm:CPG_imple}
Given 2 LIF neurons, the emitter and the resetter, with decay rate $\beta$, threshold $U_{thr}$, and reset potential $V_{reset}$, starting with resting potential $U_e(0)=U_r(0)=0$. If 
\begin{align}
    I_e(t) &= I_{c1}+S_e(t-\Delta t)(U_{thr}-I_{c1}-V_{reset})-S_r(t-\Delta t)U_{thr},\\
    I_r(t) &= S_e(t-\Delta t)I_{c2}-S_r(t-\Delta t)(I_{c2}+V_{reset}),\\
    I_{c1} &= \frac{U_{thr}\beta(1-\beta)}{\beta-\beta^R},\\
    I_{c2} &= \frac{U_{thr}\beta(1-\beta)}{\beta-\beta^{K-1}},
\end{align}
then the system will have the period of $T=(R+K)\Delta t$, and $\forall i\in\mathbb{N}\cap[0,R+K-1], k\in\mathbb{N}$:
\begin{align}
     S_e(i\Delta t + k T) = \begin{cases}
        0, & 0\leq i < R, \\ 
        1, & R\leq i < R+K.
\end{cases}
\end{align}
\end{theorem}
\begin{proof}
Assuming the first spike generated by the emitter emits at time step $T_1$. For every $0\leq t < T_1$, we have:
\begin{align}
S_e(t) &= S_r(t) = 0, \\
I_e(t) &= I_{c1}, I_r(t) = 0.
\end{align}
Since the input current of the emitter is a constant, by \Cref{lem:LIF_spike}, we immediately get:
\begin{align}
    T_1 = \left\lceil \log_\beta(\beta-\frac{U_{thr}\beta(1-\beta)}{I_{c1}})\right\rceil = R.
\end{align}
Starting from time step $R$, let's assume the first spike generated by the resetter emits at time step $T_2$. Then for every $T_1\leq t < T_2$, we have:
\begin{align}
    S_e(t) &= 1, S_r(t) = 0, \\
    I_e(t) &= U_{thr}-V_{reset}, I_r(t)= I_{c2}.
\end{align}
Starting from $T_1$, for the emitter, the input current allows a spike event for every time step. And the input current of the resetter is a constant. Again, by applying \Cref{lem:LIF_spike}, we can get:
\begin{align}
    T_2 = T_1 + \left\lceil \log_\beta(\beta-\frac{U_{thr}\beta(1-\beta)}{I_{c2}})\right\rceil = R + K - 1. 
\end{align}
Now Consider the state at time step $R + K$:
\begin{align}
    S_e((R+K-1)\Delta t) &= S_r((R+K-1)\Delta t) = 1, \\
    I_e((R+K)\Delta t) &= I_r((R+K)\Delta t) = -V_{reset}, \\
    U_e((R+K)\Delta t) &= U_r((R+K)\Delta t) = 0.
\end{align}
This is the same as the membrane potential at time step 0. Therefore, the system will behave periodically with period $T = (R+K)\Delta T$.
\end{proof}

\Cref{thm:CPG_imple} gives a possible CPG-PE design with 2 LIF neurons, with the emitter generating $K$ consecutive spikes every $R + K$ time steps. This demonstrates that incorporating CPG-PE into the current SNN architecture is completely bio-plausible and will not introduce any burden of redesigning hardware.

\section{Implement CPG-Linear} \label{app:cpg_linear}

\begin{figure*}[htp]
\centering
\includegraphics[width=0.99 \textwidth]{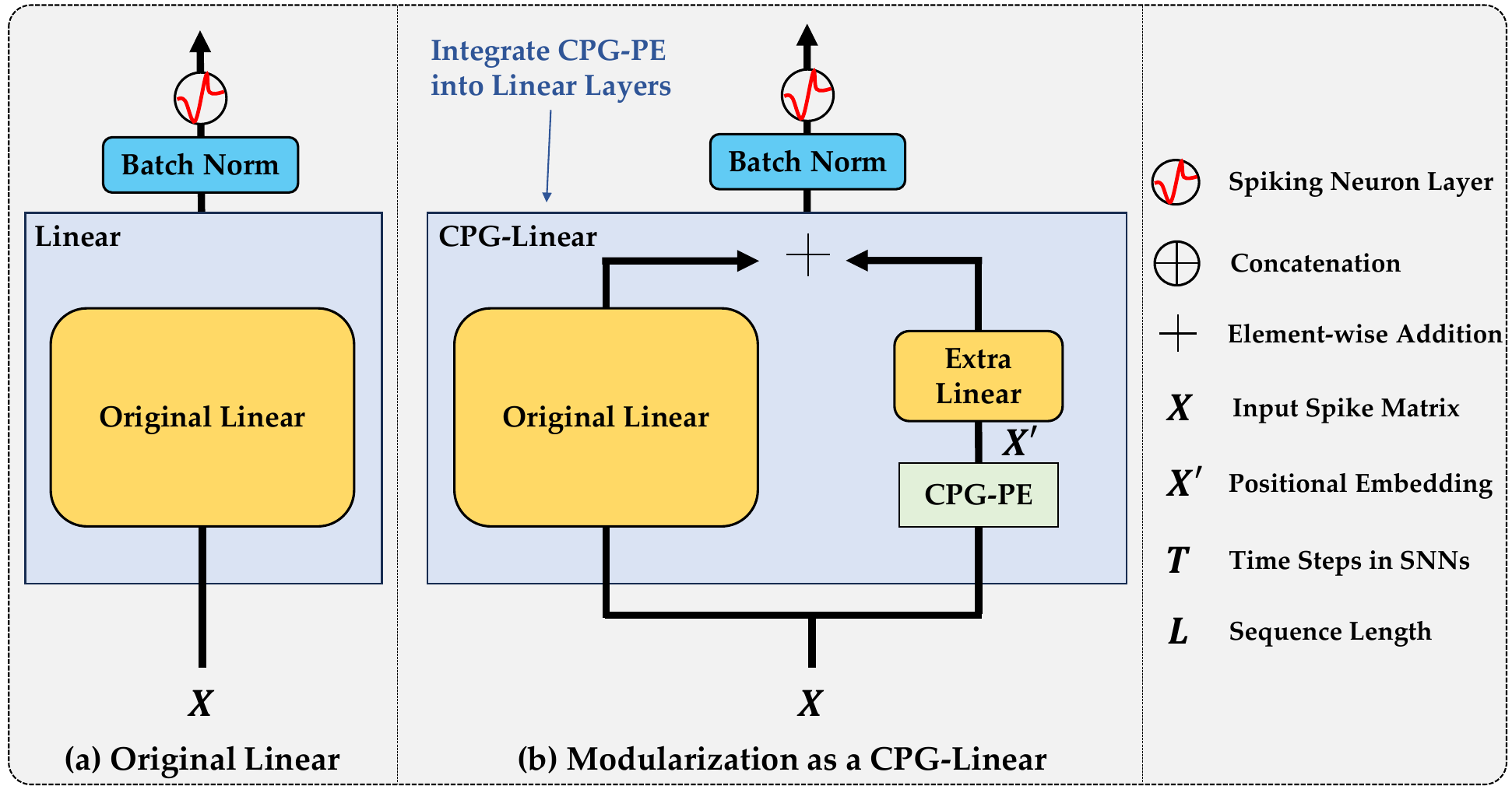}
\caption{\label{fig:cpg_linear}
An illustration of the implementation of integrating a CPG-PE into a linear layer.
}
\vspace{-3mm}
\end{figure*}

We have developed a simple modularization implementation to integrate our proposed CPG-PE with original linear layers, as depicted in \Cref{fig:implem} (b).
Consider the original linear layer's input and output dimensions as $D_{in}$ and $D_{out}$, respectively.
Our objective is to incorporate CPG-PE within this framework.
Following the application of the CPG-PE module, the modified input $X'$ is obtained.
\(X\) is then input into \(\operatorname{Linear}_{1}\) and \(X'\) into \(\operatorname{Linear}_{2}\), resulting in outputs \(X_{1}\) and \(X_{2}\), respectively. Both \(\operatorname{Linear}_{1}\) and \(\operatorname{Linear}_{2}\) maintain an output dimension of \(D_{out}\).
The final step involves summing \(X_1\) and \(X_2\) to produce \(X_3\), which is subsequently processed through batch normalization (\(\operatorname{BN}\)) and spike normalization (\(\mathcal{SN}\)).
We term this implementation as ``CPG-Linear'' and formulize as follows:
\begin{align}
& X' = \operatorname{CPG-PE}(X), && X \in \{0, 1\}^{T \times B \times L \times D_{in}}, X' \in \{0, 1\}^{T \times B \times L \times 2N}\\
& X_{1} = \operatorname{Linear_{1}} (X),  X_{2} = \operatorname{Linear_{2}}(X'), && X_{1}, X_{2} \in \mathbb{R}^{T \times B \times L \times D_{out}} \\
& X_{3} = X_{1} + X_{2}, && X_{3} \in \mathbb{R}^{T \times B \times L \times D_{out}} \\
& X_{output} = \mathcal{SN}\left(\operatorname{BN}\left(X_{3}\right)\right), && X_{out} \in \{0, 1\}^{T \times B \times L \times D_{out}}
\end{align}
where $+$ denotes element-wise addition.
This implementation described above is fundamentally identical to \Cref{fig:implem}, within the context of a single linear layer.
However, the CPG-Linear can seamlessly replace \textbf{any linear layer} in SNNs.

\section{Results on ImageNet}\label{app:imagenet}

We have conducted experiments with Spikformer without positional encoding (PE), Spikformer with relative positional encoding (RPE), and Spikformer with our proposed CPG-PE on the ImageNet dataset.
The results are as follows:

\begin{table}[ht]
\caption{
\label{tab:image_classification}
Evaluation on ImageNet benchmarks.
We employed $8$ encoder blocks and $384$ feature embedding dimensions across all models.
}
\centering
\resizebox{0.7\linewidth}{!}{
\begin{tabular}{l:cc:cc}
\toprule
\multirow{2}{*}{\textbf{Model}} & \multirow{2}{*}{\textbf{SNN}} & \multirow{2}{*}{\textbf{Spike PE}} & \multicolumn{2}{c}{\bf ImageNet} \\
& & & Param (M) & Accuracy \\
\hline
Spikformer w/o PE & {\color{red} \cmark} & -- & $15.50$ & $69.46$\\
Spikformer w/ RPE \citep{Zhou2022SpikformerWS} & {\color{red} \cmark} & {\color{red} \cmark} & $16.81$ & $70.24$\\
\rowcolor{cpgcolor}
Spikformer w/ CPG-PE [Ours] & {\color{red} \cmark} & {\color{red} \cmark}  & $15.66$ & $\bm{71.17}$\\

\bottomrule
\end{tabular}
}
\end{table}

Specifically, we set the depth to $8$ and the dimension of representation to $384$.
From the table, we can see that CPG-PE performs well on large-scale image datasets.
We believe that the above results demonstrate the effectiveness of our proposed CPG-PE in positional encoding.

\section{Limitations and Future Works}\label{app:lim}
In this section, we will discuss the limitations and future works of our paper.

\subsection{Limitations}

As mentioned in \Cref{sec:method_implem}, our CPG-PE can not be directly applied to those SNNs where spike matrices do not have a ``sequence length'' dimension.
Our CPG-PE is optimized for processing sequential data, making it ideal for applications involving time series or natural language.
This intrinsic design, however, does not naturally extend to image data, which typically benefits from direct convolutional operations that capture spatial relationships across the entire image dimensions—height and width.
In contrast, CPG-PE requires the segmentation of images into patches, a method inspired by the Vision Transformer.
This adaptation contrasts with approaches like the Convolutional 2D layer, which applies convolution operations directly across the height and width of an image without requiring segmentation into smaller, discrete patches.
The necessity to adapt CPG-PE for image data through patching can introduce complexities and potential performance bottlenecks, as it may not effectively capture the continuous spatial relationships and local features in the image, which are crucial for tasks such as object recognition and scene understanding.

\subsection{Future Works}

To enhance the applicability of the CPG-PE model to a broader range of data types, especially image data, future research could focus on developing a hybrid model that integrates the strengths of CPG-PE with traditional convolutional layers.
This integration could potentially allow the model to handle both sequential and spatial data efficiently without the need for extensive pre-processing or adaptation.
Specifically, integrating direct convolution operations that work across the entire spatial dimensions of an image within the CPG-PE architecture could help preserve spatial relationships and improve feature extraction capabilities.
Additionally, exploring the use of adaptive patch sizes or dynamically adjusting the patching mechanism based on the nature of the input data could also provide a more flexible and performance-optimized approach.
These advancements would make the model more versatile and capable of tackling a wider array of tasks across different domains.

Additionally, considering that CPG-PE is an absolute positional encoding designed for SNNs, it could be beneficial to explore the potential of implementing learnable relative positional encodings in SNNs.
Such encodings would need to be developed to meet specific criteria: they must maintain the spike-form characteristic essential to SNNs and ensure the uniqueness of each position’s encoding.
This approach could significantly enhance the model’s ability to capture and utilize the temporal dynamics of input data more effectively, potentially leading to more nuanced and context-aware processing capabilities.  
Exploring adaptive patch sizes or dynamically adjusting the patching mechanism based on the nature of the input data could also provide a more flexible and performance-optimized approach.
These advancements would make the model more versatile and capable of tackling a wider array of tasks across different domains.

\newpage

\section*{NeurIPS Paper Checklist}

\begin{enumerate}

\item {\bf Claims}
    \item[] Question: Do the main claims made in the abstract and introduction accurately reflect the paper's contributions and scope?
    \item[] Answer: \answerYes{} 
    \item[] Justification: We have clarified our claims in the abstract and introduction.
    \item[] Guidelines:
    \begin{itemize}
        \item The answer NA means that the abstract and introduction do not include the claims made in the paper.
        \item The abstract and/or introduction should clearly state the claims made, including the contributions made in the paper and important assumptions and limitations. A No or NA answer to this question will not be perceived well by the reviewers. 
        \item The claims made should match theoretical and experimental results, and reflect how much the results can be expected to generalize to other settings. 
        \item It is fine to include aspirational goals as motivation as long as it is clear that these goals are not attained by the paper. 
    \end{itemize}

\item {\bf Limitations}
    \item[] Question: Does the paper discuss the limitations of the work performed by the authors?
    \item[] Answer: \answerYes{} 
    \item[] Justification: We have discussed the limitations and future work in \Cref{app:lim}.
    \item[] Guidelines:
    \begin{itemize}
        \item The answer NA means that the paper has no limitation while the answer No means that the paper has limitations, but those are not discussed in the paper. 
        \item The authors are encouraged to create a separate "Limitations" section in their paper.
        \item The paper should point out any strong assumptions and how robust the results are to violations of these assumptions (e.g., independence assumptions, noiseless settings, model well-specification, asymptotic approximations only holding locally). The authors should reflect on how these assumptions might be violated in practice and what the implications would be.
        \item The authors should reflect on the scope of the claims made, e.g., if the approach was only tested on a few datasets or with a few runs. In general, empirical results often depend on implicit assumptions, which should be articulated.
        \item The authors should reflect on the factors that influence the performance of the approach. For example, a facial recognition algorithm may perform poorly when image resolution is low or images are taken in low lighting. Or a speech-to-text system might not be used reliably to provide closed captions for online lectures because it fails to handle technical jargon.
        \item The authors should discuss the computational efficiency of the proposed algorithms and how they scale with dataset size.
        \item If applicable, the authors should discuss possible limitations of their approach to address problems of privacy and fairness.
        \item While the authors might fear that complete honesty about limitations might be used by reviewers as grounds for rejection, a worse outcome might be that reviewers discover limitations that aren't acknowledged in the paper. The authors should use their best judgment and recognize that individual actions in favor of transparency play an important role in developing norms that preserve the integrity of the community. Reviewers will be specifically instructed to not penalize honesty concerning limitations.
    \end{itemize}

\item {\bf Theory Assumptions and Proofs}
    \item[] Question: For each theoretical result, does the paper provide the full set of assumptions and a complete (and correct) proof?
    \item[] Answer: \answerYes{} 
    \item[] Justification: We have provided the full set of assumptions and a complete (and correct) proof in the Method Section.
    \item[] Guidelines:
    \begin{itemize}
        \item The answer NA means that the paper does not include theoretical results. 
        \item All the theorems, formulas, and proofs in the paper should be numbered and cross-referenced.
        \item All assumptions should be clearly stated or referenced in the statement of any theorems.
        \item The proofs can either appear in the main paper or the supplemental material, but if they appear in the supplemental material, the authors are encouraged to provide a short proof sketch to provide intuition. 
        \item Inversely, any informal proof provided in the core of the paper should be complemented by formal proofs provided in appendix or supplemental material.
        \item Theorems and Lemmas that the proof relies upon should be properly referenced. 
    \end{itemize}

    \item {\bf Experimental Result Reproducibility}
    \item[] Question: Does the paper fully disclose all the information needed to reproduce the main experimental results of the paper to the extent that it affects the main claims and/or conclusions of the paper (regardless of whether the code and data are provided or not)?
    \item[] Answer: \answerYes{} 
    \item[] Justification: We have shown our experiment results in the Experiment Section. We have submitted our source code in Supplementary Material. We will upload our code and data to Github upon acceptance.
    \item[] Guidelines:
    \begin{itemize}
        \item The answer NA means that the paper does not include experiments.
        \item If the paper includes experiments, a No answer to this question will not be perceived well by the reviewers: Making the paper reproducible is important, regardless of whether the code and data are provided or not.
        \item If the contribution is a dataset and/or model, the authors should describe the steps taken to make their results reproducible or verifiable. 
        \item Depending on the contribution, reproducibility can be accomplished in various ways. For example, if the contribution is a novel architecture, describing the architecture fully might suffice, or if the contribution is a specific model and empirical evaluation, it may be necessary to either make it possible for others to replicate the model with the same dataset, or provide access to the model. In general. releasing code and data is often one good way to accomplish this, but reproducibility can also be provided via detailed instructions for how to replicate the results, access to a hosted model (e.g., in the case of a large language model), releasing of a model checkpoint, or other means that are appropriate to the research performed.
        \item While NeurIPS does not require releasing code, the conference does require all submissions to provide some reasonable avenue for reproducibility, which may depend on the nature of the contribution. For example
        \begin{enumerate}
            \item If the contribution is primarily a new algorithm, the paper should make it clear how to reproduce that algorithm.
            \item If the contribution is primarily a new model architecture, the paper should describe the architecture clearly and fully.
            \item If the contribution is a new model (e.g., a large language model), then there should either be a way to access this model for reproducing the results or a way to reproduce the model (e.g., with an open-source dataset or instructions for how to construct the dataset).
            \item We recognize that reproducibility may be tricky in some cases, in which case authors are welcome to describe the particular way they provide for reproducibility. In the case of closed-source models, it may be that access to the model is limited in some way (e.g., to registered users), but it should be possible for other researchers to have some path to reproducing or verifying the results.
        \end{enumerate}
    \end{itemize}

\item {\bf Open access to data and code}
    \item[] Question: Does the paper provide open access to the data and code, with sufficient instructions to faithfully reproduce the main experimental results, as described in supplemental material?
    \item[] Answer: \answerYes{} 
    \item[] Justification: We have submitted our source code in Supplementary Material. We will upload our code and data to Github upon acceptance.
    \item[] Guidelines:
    \begin{itemize}
        \item The answer NA means that paper does not include experiments requiring code.
        \item Please see the NeurIPS code and data submission guidelines (\url{https://nips.cc/public/guides/CodeSubmissionPolicy}) for more details.
        \item While we encourage the release of code and data, we understand that this might not be possible, so “No” is an acceptable answer. Papers cannot be rejected simply for not including code, unless this is central to the contribution (e.g., for a new open-source benchmark).
        \item The instructions should contain the exact command and environment needed to run to reproduce the results. See the NeurIPS code and data submission guidelines (\url{https://nips.cc/public/guides/CodeSubmissionPolicy}) for more details.
        \item The authors should provide instructions on data access and preparation, including how to access the raw data, preprocessed data, intermediate data, and generated data, etc.
        \item The authors should provide scripts to reproduce all experimental results for the new proposed method and baselines. If only a subset of experiments are reproducible, they should state which ones are omitted from the script and why.
        \item At submission time, to preserve anonymity, the authors should release anonymized versions (if applicable).
        \item Providing as much information as possible in supplemental material (appended to the paper) is recommended, but including URLs to data and code is permitted.
    \end{itemize}

\item {\bf Experimental Setting/Details}
    \item[] Question: Does the paper specify all the training and test details (e.g., data splits, hyperparameters, how they were chosen, type of optimizer, etc.) necessary to understand the results?
    \item[] Answer: \answerYes{} 
    \item[] Justification: We have shown our experimental settings and implementation details in \Cref{app:datasets,app:imple} respectively.
    \item[] Guidelines:
    \begin{itemize}
        \item The answer NA means that the paper does not include experiments.
        \item The experimental setting should be presented in the core of the paper to a level of detail that is necessary to appreciate the results and make sense of them.
        \item The full details can be provided either with the code, in appendix, or as supplemental material.
    \end{itemize}

\item {\bf Experiment Statistical Significance}
    \item[] Question: Does the paper report error bars suitably and correctly defined or other appropriate information about the statistical significance of the experiments?
    \item[] Answer: \answerYes{} 
    \item[] Justification: Our reported results are all averaged over several random seeds. We have reported the error bar of the results in \Cref{tab:text_table}.
    \item[] Guidelines:
    \begin{itemize} 
        \item The answer NA means that the paper does not include experiments.
        \item The authors should answer "Yes" if the results are accompanied by error bars, confidence intervals, or statistical significance tests, at least for the experiments that support the main claims of the paper.
        \item The factors of variability that the error bars are capturing should be clearly stated (for example, train/test split, initialization, random drawing of some parameter, or overall run with given experimental conditions).
        \item The method for calculating the error bars should be explained (closed form formula, call to a library function, bootstrap, etc.)
        \item The assumptions made should be given (e.g., Normally distributed errors).
        \item It should be clear whether the error bar is the standard deviation or the standard error of the mean.
        \item It is OK to report 1-sigma error bars, but one should state it. The authors should preferably report a 2-sigma error bar than state that they have a 96\% CI, if the hypothesis of Normality of errors is not verified.
        \item For asymmetric distributions, the authors should be careful not to show in tables or figures symmetric error bars that would yield results that are out of range (e.g. negative error rates).
        \item If error bars are reported in tables or plots, The authors should explain in the text how they were calculated and reference the corresponding figures or tables in the text.
    \end{itemize}

\item {\bf Experiments Compute Resources}
    \item[] Question: For each experiment, does the paper provide sufficient information on the computer resources (type of compute workers, memory, time of execution) needed to reproduce the experiments?
    \item[] Answer: \answerYes{} 
    \item[] Justification: We have provided the compute resource in \Cref{app:imple}.
    \item[] Guidelines:
    \begin{itemize}
        \item The answer NA means that the paper does not include experiments.
        \item The paper should indicate the type of compute workers CPU or GPU, internal cluster, or cloud provider, including relevant memory and storage.
        \item The paper should provide the amount of compute required for each of the individual experimental runs as well as estimate the total compute. 
        \item The paper should disclose whether the full research project required more compute than the experiments reported in the paper (e.g., preliminary or failed experiments that didn't make it into the paper). 
    \end{itemize}
    
\item {\bf Code Of Ethics}
    \item[] Question: Does the research conducted in the paper conform, in every respect, with the NeurIPS Code of Ethics \url{https://neurips.cc/public/EthicsGuidelines}?
    \item[] Answer: \answerYes{} 
    \item[] Justification: The research conducted in the paper conforms, in every respect, with the NeurIPS Code of Ethics.
    \item[] Guidelines:
    \begin{itemize}
        \item The answer NA means that the authors have not reviewed the NeurIPS Code of Ethics.
        \item If the authors answer No, they should explain the special circumstances that require a deviation from the Code of Ethics.
        \item The authors should make sure to preserve anonymity (e.g., if there is a special consideration due to laws or regulations in their jurisdiction).
    \end{itemize}

\item {\bf Broader Impacts}
    \item[] Question: Does the paper discuss both potential positive societal impacts and negative societal impacts of the work performed?
    \item[] Answer: \answerYes{} 
    \item[] Justification: We have discussed both potential positive societal impacts and negative societal impacts of the work in Broader Impact Section.
    \item[] Guidelines:
    \begin{itemize}
        \item The answer NA means that there is no societal impact of the work performed.
        \item If the authors answer NA or No, they should explain why their work has no societal impact or why the paper does not address societal impact.
        \item Examples of negative societal impacts include potential malicious or unintended uses (e.g., disinformation, generating fake profiles, surveillance), fairness considerations (e.g., deployment of technologies that could make decisions that unfairly impact specific groups), privacy considerations, and security considerations.
        \item The conference expects that many papers will be foundational research and not tied to particular applications, let alone deployments. However, if there is a direct path to any negative applications, the authors should point it out. For example, it is legitimate to point out that an improvement in the quality of generative models could be used to generate deepfakes for disinformation. On the other hand, it is not needed to point out that a generic algorithm for optimizing neural networks could enable people to train models that generate Deepfakes faster.
        \item The authors should consider possible harms that could arise when the technology is being used as intended and functioning correctly, harms that could arise when the technology is being used as intended but gives incorrect results, and harms following from (intentional or unintentional) misuse of the technology.
        \item If there are negative societal impacts, the authors could also discuss possible mitigation strategies (e.g., gated release of models, providing defenses in addition to attacks, mechanisms for monitoring misuse, mechanisms to monitor how a system learns from feedback over time, improving the efficiency and accessibility of ML).
    \end{itemize}
    
\item {\bf Safeguards}
    \item[] Question: Does the paper describe safeguards that have been put in place for responsible release of data or models that have a high risk for misuse (e.g., pretrained language models, image generators, or scraped datasets)?
    \item[] Answer: \answerNA{} 
    \item[] Justification: The paper poses no such risks.
    \item[] Guidelines:
    \begin{itemize}
        \item The answer NA means that the paper poses no such risks.
        \item Released models that have a high risk for misuse or dual-use should be released with necessary safeguards to allow for controlled use of the model, for example by requiring that users adhere to usage guidelines or restrictions to access the model or implementing safety filters. 
        \item Datasets that have been scraped from the Internet could pose safety risks. The authors should describe how they avoided releasing unsafe images.
        \item We recognize that providing effective safeguards is challenging, and many papers do not require this, but we encourage authors to take this into account and make a best faith effort.
    \end{itemize}

\item {\bf Licenses for existing assets}
    \item[] Question: Are the creators or original owners of assets (e.g., code, data, models), used in the paper, properly credited and are the license and terms of use explicitly mentioned and properly respected?
    \item[] Answer: \answerYes{} 
    \item[] Justification: The datasets we used in the paper are all public datasets. Please refer to \Cref{app:datasets} for details of datasets.
    \item[] Guidelines:
    \begin{itemize}
        \item The answer NA means that the paper does not use existing assets.
        \item The authors should cite the original paper that produced the code package or dataset.
        \item The authors should state which version of the asset is used and, if possible, include a URL.
        \item The name of the license (e.g., CC-BY 4.0) should be included for each asset.
        \item For scraped data from a particular source (e.g., website), the copyright and terms of service of that source should be provided.
        \item If assets are released, the license, copyright information, and terms of use in the package should be provided. For popular datasets, \url{paperswithcode.com/datasets} has curated licenses for some datasets. Their licensing guide can help determine the license of a dataset.
        \item For existing datasets that are re-packaged, both the original license and the license of the derived asset (if it has changed) should be provided.
        \item If this information is not available online, the authors are encouraged to reach out to the asset's creators.
    \end{itemize}

\item {\bf New Assets}
    \item[] Question: Are new assets introduced in the paper well documented and is the documentation provided alongside the assets?
    \item[] Answer: \answerNA{} 
    \item[] Justification: The paper does not release new assets.
    \item[] Guidelines:
    \begin{itemize}
        \item The answer NA means that the paper does not release new assets.
        \item Researchers should communicate the details of the dataset/code/model as part of their submissions via structured templates. This includes details about training, license, limitations, etc. 
        \item The paper should discuss whether and how consent was obtained from people whose asset is used.
        \item At submission time, remember to anonymize your assets (if applicable). You can either create an anonymized URL or include an anonymized zip file.
    \end{itemize}

\item {\bf Crowdsourcing and Research with Human Subjects}
    \item[] Question: For crowdsourcing experiments and research with human subjects, does the paper include the full text of instructions given to participants and screenshots, if applicable, as well as details about compensation (if any)? 
    \item[] Answer: \answerNA{} 
    \item[] Justification: The paper does not involve crowdsourcing nor research with human subjects.
    \item[] Guidelines:
    \begin{itemize}
        \item The answer NA means that the paper does not involve crowdsourcing nor research with human subjects.
        \item Including this information in the supplemental material is fine, but if the main contribution of the paper involves human subjects, then as much detail as possible should be included in the main paper. 
        \item According to the NeurIPS Code of Ethics, workers involved in data collection, curation, or other labor should be paid at least the minimum wage in the country of the data collector. 
    \end{itemize}

\item {\bf Institutional Review Board (IRB) Approvals or Equivalent for Research with Human Subjects}
    \item[] Question: Does the paper describe potential risks incurred by study participants, whether such risks were disclosed to the subjects, and whether Institutional Review Board (IRB) approvals (or an equivalent approval/review based on the requirements of your country or institution) were obtained?
    \item[] Answer: \answerNA{} 
    \item[] Justification: The paper does not involve crowdsourcing nor research with human subjects.
    \item[] Guidelines:
    \begin{itemize}
        \item The answer NA means that the paper does not involve crowdsourcing nor research with human subjects.
        \item Depending on the country in which research is conducted, IRB approval (or equivalent) may be required for any human subjects research. If you obtained IRB approval, you should clearly state this in the paper. 
        \item We recognize that the procedures for this may vary significantly between institutions and locations, and we expect authors to adhere to the NeurIPS Code of Ethics and the guidelines for their institution. 
        \item For initial submissions, do not include any information that would break anonymity (if applicable), such as the institution conducting the review.
    \end{itemize}

\end{enumerate}

\end{document}